\documentclass[11pt,letterpaper]{article}
\usepackage[margin=.8in]{geometry}

\usepackage{microtype}
\usepackage{graphicx}
\usepackage{subcaption}
\usepackage{booktabs}
\usepackage{tablefootnote}
\usepackage{footmisc}
\usepackage{multirow}
\usepackage{bbm}
\usepackage{hyperref}
\usepackage[ruled,vlined]{algorithm2e}
\usepackage{algpseudocode}

\usepackage{amsmath}
\usepackage{amssymb}
\usepackage{mathtools}
\usepackage{amsthm}
\usepackage{enumitem}
\usepackage[capitalize,noabbrev]{cleveref}

\theoremstyle{plain}
\newtheorem{theorem}{Theorem}

\newtheorem{lemma}{Lemma}
\newtheorem{corollary}{Corollary}
\theoremstyle{definition}
\newtheorem{definition}{Definition}
\newtheorem{assumption}{Assumption}
\theoremstyle{remark}

\newcommand{\Ac}{\mathcal{A}}

\newcommand{\Dc}{\mathcal{D}}
\newcommand{\Ec}{\mathcal{E}}

\newcommand{\Hc}{\mathcal{H}}

\newcommand{\Oc}{\mathcal{O}}

\newcommand{\Sc}{\mathcal{S}}

\newcommand{\Xc}{\mathcal{X}}

\newcommand{\Pb}{\mathbb{P}}

\newcommand{\Eb}{\mathbb{E}}
\newcommand{\Rb}{\mathbb{R}}
\newcommand{\Zb}{\mathbb{Z}}

\newcommand{\argmax}{\arg\max}
\newcommand{\argmin}{\arg\min}

\newcommand{\diag}{\mathrm{diag}}

\usepackage[textsize=tiny]{todonotes}

\date{}

\begin{document}
\title{Policy Optimization for Constrained MDPs with \\ Provable Fast Global Convergence}

\author{
  Tao Liu\thanks{The first two authors contributed equally.  Email: \texttt{\{tliu, ruida, dileep.kalathil, prk, chao.tian\}@tamu.edu }}
  \quad
  Ruida Zhou$^*$
  \quad
  Dileep Kalathil
  \quad
  P. R. Kumar
  \quad
  Chao Tian \\
Department of Electrical and Computer Engineering\\ 
Texas A\&M University
}

\maketitle

\begin{abstract}
We address the problem of finding the optimal policy of a constrained Markov decision process (CMDP) using a gradient descent-based algorithm. Previous results have shown that a primal-dual approach can achieve an $\mathcal{O}(1/\sqrt{T})$ global convergence rate for both the optimality gap and the constraint violation. We propose a new algorithm called policy mirror descent-primal dual (PMD-PD) algorithm that can provably achieve a faster $\mathcal{O}(\log(T)/T)$ convergence rate for both the optimality gap and the constraint violation. For the primal (policy) update, the PMD-PD algorithm utilizes a modified value function and performs natural policy gradient steps, which is equivalent to a mirror descent step with appropriate regularization. For the dual update, the PMD-PD algorithm uses modified Lagrange multipliers to ensure a faster convergence rate. We also present two extensions of this approach to the settings with zero constraint violation and sample-based estimation. Experimental results demonstrate the faster convergence rate and the better performance of the PMD-PD algorithm compared with existing policy gradient-based algorithms. 
\end{abstract}

\section{Introduction}
Policy gradient (PG) methods and their variants play an important role in reinforcement learning (RL). The gradient-based methods are attractive due to their flexibility in being applicable to any differentiable policy parameterization and generality for extensions to function approximation of policies \cite{agarwal2021theory}. Standard PG methods have been applied to Markov decision processes (MDPs), which focus on optimizing a single objective without any explicit constraint on the policies. However, in many real-world applications, stringent safety constraints are imposed on control policies \cite{dulac2021challenges, amodei2016concrete, garcia2015comprehensive}. For example, a mobile wireless application may desire to maximize throughput, with a constraint on power consumption. The model of a constrained Markov decision process (CMDP) \cite{altman1999constrained}, where the goal is to optimize an objective while satisfying safety constraints, is a standard approach for modeling the necessary safety criteria of a control problem through constraints on safety costs. 

Recently, many algorithms using PG and natural policy gradient (NPG) \cite{kakade2001natural} methods have been developed for solving CMDPs.  Lagrangian-based methods \cite{tessler2018reward, stooke2020responsive, paternain2019safe, liang2018accelerated} optimize CMDP as a saddle-point problem via primal-dual approach, while constrained policy optimization methods \cite{achiam2017constrained, yang2019projection} calculate new dual variables from scratch at each update to maintain constraints during the learning process. Although these algorithms provide a way to iteratively optimize the learned policy, they can only guarantee local convergence, with no guarantees on the convergence rate to a globally optimal solution.

Motivated by recent works \cite{agarwal2020optimality, mei2020global} that prove global convergence of PG algorithms for MDPs, Ding et al. \cite{ding2020natural} developed an NPG-based primal-dual method with the softmax policy parameterization, which provides an $\Oc(1/\sqrt{T})$ global convergence rate for both the optimality gap and the constraint violation, where $T$ is the total number of iterations that the algorithm executes. Similarly, Xu et al. \cite{xu2021crpo} showed how to attain the global optimal with the same order convergence rate via NPG-based primal methods. However, it is known that the NPG-based methods for unconstrained MDPs enjoy $\Oc(1/T)$ convergence rate \cite{agarwal2021theory} and $\exp(-T)$ convergence rate with certain regularizations \cite{cen2021fast, zhan2021policy}. This motivates the following important theoretical question:

\textbf{Can we design a policy gradient-based  algorithm for CMDPs that can provably achieve a global convergence rate faster than $\Oc(1/\sqrt{T})$?}

\vspace{-0.1in}
\paragraph{Our contribution:} We answer the above question affirmatively by proposing a new algorithm, which we call policy mirror descent-primal dual (PMD-PD) algorithm. We show that the PMD-PD algorithm achieves an $\Oc(\log(T)/T)$ global convergence rate for both the optimality gap and the constraint violation. For the primal (policy) update, the PMD-PD algorithm utilizes a modified value function and performs natural policy gradient steps, which is equivalent to a mirror descent step with appropriate regularization. For the dual update, the PMD-PD algorithm uses a modified Lagrange multiplier to ensure a faster convergence rate. The PMD-PD algorithm is nearly dimension-free (depending at most logarithmically on the dimension of the action space) and is faster than existing PG-based algorithms for CMDPs (listed in Table \ref{tab:compare}). We also present an extension called the PMD-PD-Zero algorithm that can return policies with zero constraint violation without compromising the order of the convergence rate for the optimality gap. Additionally, we extend the PMD-PD algorithm to the sample-based setting (without an oracle for exact policy evaluation) and show an $\tilde{\Oc}(1/\epsilon^3)$ sample complexity, which is more efficient compared with a sample complexity of $\Oc(1/\epsilon^4)$ from existing PG-based algorithms for CMDPs \cite{xu2021crpo}. 

\begin{table}[t]
\caption{Global convergence rates for the optimality gap and the constraint violation with exact policy evaluation. All algorithms listed  are under the softmax policy parameterization.}
\label{tab:compare}
    \begin{center}
    \begin{small}
    \begin{sc}
    \begin{tabular}{lll}
    \toprule
    Algorithm & Optimality gap \tablefootnote{This table is presented for $T \ge \text{poly}(|\Sc|, |\Ac|, \|d_{\rho}^{\pi^*}/\nu\|_{\infty})$, with polynomial terms independent of $T$ omitted, where $|\Sc|$ and $|\Ac|$ are the number of states and actions respectively, $\nu$ is the starting state distribution for the algorithms, and $d_{\rho}^{\pi^*}$ is the state visitation distribution when executing an optimal policy $\pi^*$. \label{omitpoly}} & Violation \footref{omitpoly}\\
    \midrule
    PG/NPG \cite{mei2020global, agarwal2021theory} & $\Oc(1/T)$ & / \\
    PG-Entropy \cite{mei2020global} & $\Oc(\exp(-T))$ & / \\
    NPG-Entropy \cite{cen2021fast} & $\Oc(\exp(-T))$ & / \\
    NPG-PD \cite{ding2020natural} & $\Oc(1/\sqrt{T})$ & $\Oc(1/\sqrt{T})$\\
    CRPO \cite{xu2021crpo} & $\Oc(1/\sqrt{T})$ & $\Oc(1/\sqrt{T})$\\
    \midrule
    \textbf{PMD-PD} & ${\Oc}(\log(T)/T)$  & ${\Oc}(\log(T)/T)$ \\
    \textbf{PMD-PD-Zero} & ${\Oc}(\log(T)/T)$ & 0 \tablefootnote{It holds after some $T$. 
    Details are provided in Section \ref{sec:zero}.} \\
    \bottomrule
    \end{tabular}
    \end{sc}
    \end{small}
    \end{center}
\end{table}

\subsection{Related Work}

\paragraph{Global convergence of PG algorithms:}
Recently, there has been much study of the global convergence properties of policy gradient methods. It has been shown that PG and NPG methods can achieve $\Oc(1/T)$ convergence \cite{mei2020global, agarwal2021theory} for unregularized MDPs. When entropy regularization is used, both PG and NPG methods can guarantee $\text{exp}(-T)$ convergence \cite{mei2020global, cen2021fast} to the optimal solution of the regularized problem.  NPG methods can be interpreted as  mirror descent \cite{neu2017unified, zhan2021policy}, thereby enabling the adaptation of mirror descent techniques to analyze NPG-based methods.

The global convergence analysis of the PG methods for MDPs has also been extended to CMDPs. \cite{ding2020natural} proposed the NPG-PD algorithm which uses a primal-dual approach with NPG and showed that it can achieve $\Oc(1/\sqrt{T})$ global convergence for both the optimality gap and the constraint violation. \cite{xu2021crpo} proposed a primal approach called constrained-rectified policy optimization (CRPO), which updates the policy alternatively between optimizing objective and decreasing constraint violation, and enjoys the same $\Oc(1/\sqrt{T})$ global convergence. Our work focuses on achieving a faster convergence rate for the CMDP problem, motivated by the results for MDPs with a convergence rate faster than $\Oc(1/\sqrt{T})$ (see Table \ref{tab:compare}).

In the work conducted concurrently with ours, but with different results, Ying et al. \cite{ying2021dual} and Li et al. \cite{li2021faster} address the same question of developing PG-based algorithms for the CMDP problem. Ying et al. \cite{ying2021dual} propose an NPG-aided dual approach, where the dual function is smoothed by entropy regularization in the objective function. They show an $\tilde{\Oc}(1/T)$ convergence rate to the optimal policy of the \textit{entropy-regularized} CMDP, but not to the true optimal policy, for which with a slow $\Oc(1/\sqrt{T})$ convergence rate. They also make an additional strong assumption that the initial state distribution covers the entire state space. While such an assumption was initially used in the analysis of the global convergence of PG methods for MDPs \cite{agarwal2021theory, mei2020global}, it is not required when analyzing the global convergence of NPG methods \cite{agarwal2021theory, cen2021fast}. Moreover, this assumption does not necessarily hold for safe RL or CMDP, since the algorithm needs to avoid dangerous states even at initialization and the optimal policy will depend on the initial state distribution. Li et al. \cite{li2021faster} propose a primal-dual approach with an $\Oc(\log^2(T)/T)$ convergence rate to the true optimal policy by smoothing the Lagrangian with suitable regularization on both primal and dual variables. However, they assume that the Markov chain induced by any stationary policy is ergodic in order to ensure the smoothness of the dual function. This assumption, though weaker than the assumption made by \cite{ying2021dual}, will generally not hold in problems where one wants to avoid unsafe states altogether. In this work, we propose an algorithm with a faster $\Oc(\log(T)/T)$ convergence rate to the true optimal policy without such assumptions. Moreover, we also present two important extensions of our approach to the settings with zero constraint violation and sample-based estimation.

\vspace{-0.1in}
\paragraph{Fast convergence of constrained convex optimization:}
The conventional primal-dual subgradient method used to solve convex optimization with functional constraints has a convergence rate lower bounded by $\Omega(1/\sqrt{T})$ \cite{bubeck2015convex}. Assuming access to a proximal mapping, Yu and Neely \cite{yu2017simple} proposed a new Lagrangian dual algorithm with an $\Oc(1/T)$ convergence rate by augmenting the Lagrange multipliers. Under the smoothness assumption, the same $\Oc(1/T)$ convergence rate can be attained without needing access to a proximal mapping \cite{yu2017primal}. In our work, we adapt some of the techniques introduced in \cite{yu2017primal} to the CMDP setting. Recently, Xu \cite{xu2020first} pointed out that $\Oc(\exp(-T))$ convergence rate can be attained if the objective function possesses an additional strong-convexity property and there are only a bounded number of constraints. 
\vspace{-0.1in}
\paragraph{Notations:} For any given set $\Xc$,  $\Delta(\Xc)$ denotes the probability simplex over the set $\Xc$, and $|\Xc|$ denotes the cardinality of the set $\Xc$. For any $p_{1}, p_{2} \in \Delta(\Xc)$, the Kullback–Leibler (KL) divergence between $p_{1}$ and $p_{2}$ is defined as $D(p_{1} || p_{2}) := \sum_{x \in \Xc} p_{1}(x) \log \frac{p_{1}(x)}{p_{2}(x)}$, where the logarithm is base $e$. For any integer $m$, $[m] := \{1, \dots, m\}$. For any $a \in \mathbb{R}$, $(a)_+ := \max\{a, 0\}$, $\lceil a \rceil := \min\{n \in \Zb ~|~ n \geq a\}$.

\section{Preliminaries}
\subsection{Problem Formulation}
A discounted infinite-horizon CMDP model is a tuple $M = (\Sc, \Ac, P, c_0, c_{1:m}, \rho, \gamma)$, where $\Sc$ is the state space, $\Ac$ is the action space, $c_0: \Sc \times \Ac \to [-1, 1]$ is the objective cost function, $c_i: \Sc \times \Ac \to [-1, 1]$ is the $i$-th constraint cost function, for $i \in [m]$, $P: \Sc \times \Ac \to \Delta(\Sc)$ is the transition kernel, $\rho \in \Delta(\Sc)$ is the starting state distribution over $\Sc$, and $\gamma \in [0, 1)$ is the discount factor.  Given any stationary randomized policy $\pi: \Sc \to \Delta(\Ac)$ and any cost function $c : \Sc \times \Ac \to [-1, 1]$, we define the state value function $V_{c}^{\pi}$ and  and the state-action value function $Q_{c}^{\pi}$ as $V_{c}^{\pi}(s) := \Eb [\sum_{t=0}^{\infty} \gamma^t c(s_t, a_t) ~|~ s_0=s, \pi ],~~Q_{c}^{\pi}(s, a) := \Eb [\sum_{t=0}^{\infty} \gamma^t c(s_t, a_t) ~|~ s_0=s, a_0=a, \pi ],$ where the expectation $\Eb$ is taken over the randomness of the trajectory of the Markov chain induced by policy $\pi$ and transition kernel $P$. With slight abuse of notation, denote $V^{\pi}_c(\rho) := \Eb_{s \sim \rho}[ V^{\pi}_c(s) ]$.

For any policy $\pi$ and any $(s, a) \in \Sc \times \Ac$, we define the  discounted state-action visitation distribution  as  $d_{\rho}^{\pi}(s, a) := (1 - \gamma) \Eb_{s_0 \sim \rho}[\sum_{t=0}^{\infty} \gamma^t \Pb(s_t=s, a_t=a|s_0)]$.  It then follows that $V^{\pi}_c(\rho) = \frac{1}{1 - \gamma} \langle d_{\rho}^{\pi}, c \rangle$ by viewing $d^{\pi}_\rho$ and $c$ as $|\Sc||\Ac|$-dimensional vectors indexed by $(s, a) \in \Sc \times \Ac$. When it is clear from the context, with slight abuse of notation, we also denote the discounted state visitation distribution with respect to (w.r.t.) the initial state distribution $\rho$ and policy $\pi$ by $d_{\rho}^{\pi}(s) := (1 - \gamma) \Eb_{s_0 \sim \rho}[ \sum_{t=0}^{\infty} \gamma^t \Pb(s_t=s|s_0)], \ \forall s \in \Sc$. Note that $d^{\pi}_\rho(s) = \sum_{a \in \Ac} d^{\pi}_\rho(s, a)$. For any two policies $\pi, \pi'$ and for any discounted state visitation distribution $d$, the expected KL divergence  between $\pi$ and  $\pi'$ is defined as $D_{d}(\pi || \pi' ) := \sum_{s \in \Sc} d(s) D\left(\pi(\cdot| s)|| \pi'(\cdot | s)\right)$. 

Given a CMDP $M$, the  goal is to solve the constrained optimization problem:
\begin{align}
    \min_{\pi}~~ V_{c_0}^{\pi}(\rho) \label{eqn:cmdp},\quad \text{s.t.}~~ V_{c_i}^{\pi}(\rho) \leq 0,\quad \forall i \in [m].
\end{align}

Let $\pi^*$ be the optimal policy of the CMDP problem in (\ref{eqn:cmdp}). It is well-known that in general the optimal policy $\pi^*$ is randomized and the Bellman equation may not hold \cite{altman1999constrained}. We assume strict feasibility of \eqref{eqn:cmdp}, which naturally implies the existence of the optimal policy.

\begin{assumption}[Slater's condition] \label{asm:slater}
There exists $\xi > 0$ and $\overline{\pi}$ such that $V_{c_i}^{\overline{\pi}}(\rho) \le - \xi$, $\forall i \in [m]$.
\end{assumption}

This assumption is quite standard in the optimization literature for analyzing primal-dual algorithms \cite{bertsekas2014constrained}. In particular, many related works in the CMDP literature (see, e.g., \cite{ding2020natural, ding2021provably, efroni2020exploration, liu2021learning}) make the same strict feasibility assumption. Note that unlike previous primal-dual algorithms \cite{ding2020natural, ding2021provably} for CMDPs, where $\xi$ is required to be known a priori for the projection of dual variables, our proposed algorithm does not require the knowledge of $\xi$, and this assumption is made only for the analysis.

The constrained optimization problem in (\ref{eqn:cmdp}) can be reparameterized by using the discounted state-action visitation distribution as decision variables, as follows \cite{altman1999constrained}:
\begin{align}
     \min_{d \in \Dc} ~~ \frac{1}{1 - \gamma} \langle d, c_0 \rangle \label{eqn:LP-CMDP} \quad \text{s.t.}~~ \frac{1}{1-\gamma} \langle d, c_i\rangle \leq 0, \quad \forall i \in [m],
\end{align}
where $\Dc$ is the domain of visitation distributions defined as $ \Dc := \{d \in \Delta(\Sc \times \Ac):  \gamma \sum_{s', a'} P(s | s', a') d(s', a') +  (1 - \gamma) \rho(s) =  \sum_{a} d(s, a), ~ \forall s \in \Sc \}$. It is straight forward to notice that $\Dc$ is a compact convex set, and the linear programming (LP) formulation of the CMDP problem in (\ref{eqn:LP-CMDP}) satisfies strong duality.

The LP approach can be computationally expensive for CMDPs with a large number of states and actions. Moreover, the LP approach requires explicit knowledge of the transition kernel $P$, which makes it not amenable to model-free RL algorithms. In this work, we focus on a policy gradient-based approach for solving the CMDP problem.

\subsection{Gradient-based Approach for Solving CMDPs}  
For the constrained optimization problem in (\ref{eqn:cmdp}), define its Lagrangian as
\begin{align*}
    L(\pi, \lambda) := V_{c_0}^{\pi}(\rho) + \sum_{i=1}^m \lambda_i V_{c_i}^{\pi}(\rho),
\end{align*}
where $\lambda_i$ is the Lagrange multiplier corresponding to the $i$-th constraint, for each $ i \in [m]$. Due to its equivalence to the LP formulation in (\ref{eqn:LP-CMDP}) and the consequent strong duality \cite{altman1999constrained}, the optimal value of the CMDP satisfies
\begin{align*}
    V^{\pi^*}_{c_0}(\rho) = \min_{\pi} \max_{\lambda \ge 0} L(\pi, \lambda) = \max_{\lambda \ge 0} \min_{\pi}  L(\pi, \lambda).
\end{align*}
Notice that for any fixed vector $\lambda \geq 0$, the Lagrangian is actually the value function of an MDP with cost $c_0 + \sum_{i=1}^m \lambda_i c_i$, i.e., $L(\pi, \lambda) = V^{\pi}_{c_0 + \sum_{i=1}^m \lambda_i c_i}(\rho)$. Algorithms for solving MDPs can therefore be applied to tackle the problem $\min_{\pi} L(\pi, \lambda)$ for any fixed $\lambda$. The Langrange dual function, defined as $G(\lambda) := \min_{\pi} L(\pi, \lambda)$, has optimal dual variables defined as $\lambda^* := \argmax_{\lambda \geq 0} G(\lambda)$. Under Assumption \ref{asm:slater}, the optimal dual variables are bounded by $\|\lambda^*\| \leq \frac{2}{\xi(1-\gamma)}$ (cf. Lemma \ref{lem:upper-dual} in the Appendix). The optimal policy satisfies $\pi^* \in \argmin_{\pi} L(\pi, \lambda^*)$. In particular, the primal-dual algorithms for solving the CMDP problem are by searching the saddle-point of its Lagrangian.

Let $\{\pi_{\theta}| \theta \in \Theta\}$ be the class of parametric policies. The PG method updates the parameter $\theta$ with learning rate $\eta$ via $\theta^{(t+1)} \leftarrow \theta^{(t)} - \eta \nabla_{\theta} V_{c_0 + \sum_{i=1}^m \lambda_i c_i}^{\pi_{\theta^{(t)}}}(\rho)$, while the NPG method uses a pre-conditioned update 
\begin{align}
\label{eqn:npg_theta}
\theta^{(t+1)} \leftarrow \theta^{(t)} - \eta F_{\rho}(\theta^{(t)})^{\dagger} \nabla_{\theta}V_{c_0 + \sum_{i=1}^m \lambda_i c_i}^{\pi_{\theta^{(t)}}}(\rho),    
\end{align}
where $F_{\rho}(\theta)^{\dagger}$ is the Moore-Penrose inverse of the Fisher information matrix defined as $F_{\rho}(\theta)^{\dagger} := \mathbb{E}_{s \sim d_{\rho}^{\pi_{\theta}}} \mathbb{E}_{a \sim \pi_{\theta}(\cdot \mid s)}\left[\nabla_{\theta} \log \pi_{\theta}(a|s)\left(\nabla_{\theta} \log \pi_{\theta}(a|s)\right)^{\top}\right]^{\dagger}$.

We focus on policies with the widely used softmax parameterization, where for  any $\theta \in \mathbb{R}^{|\Sc||\Ac|}$, we define $\pi_{\theta}$ as
\begin{align}
\label{eqn:softmax}
    \pi_{\theta}(a|s) = \frac{\exp(\theta_{s, a})}{\sum_{a' \in \Ac}\exp(\theta_{s, a'})}, \quad \forall (s, a) \in \Sc \times \Ac.
\end{align}
This policy class is differentiable and complete in the sense that it covers almost any randomized policy and its closure contains all stationary policies \cite{agarwal2021theory}.

Under the softmax parameterization (\ref{eqn:softmax}), the NPG with learning rate $\eta$ takes the form
\begin{align}
\label{eqn:softmaxNPG}
    \pi^{(t+1)}(a|s) = \pi^{(t)}(a|s) \frac{\exp(-\eta Q_{c_0 + \sum_{i=1}^m \lambda_i c_i}^{\pi^{(t)}}(s, a))}{Z_t(s)},
\end{align}
where $Z_t(s) = \sum_{a} \pi^{(t)}(a|s) \exp(-\eta Q_{c_0 + \sum_{i=1}^m \lambda_i c_i}^{\pi^{(t)}}(s, a))$. It was shown that (\ref{eqn:softmaxNPG}) is equivalent to a mirror descent update \cite{zhan2021policy}
\begin{align}
\label{eqm:mirror_descent}
    \pi^{(t+1)}(\cdot|s) = \argmin_{\pi} &\left\{\langle Q_{c_0 + \sum_{i=1}^m \lambda_i c_i}^{\pi^{(t)}}(s, \cdot), \pi(\cdot|s) \rangle + \frac{1}{\eta} D(\pi(\cdot|s) || \pi^{(t)}(\cdot|s))\right\}.
\end{align}

The conventional dual update with learning rate $\eta'$ is $\lambda_i^{(t+1)} = \min(\lambda_i^{(t)} + \eta' V_{c_i}^{\pi^{(t+1)}}(\rho), 2/((1-\gamma)\xi) )_+, \forall i \in [m].$ Ding et al. \cite{ding2020natural} used the above NPG primal-dual (PD) approach, obtaining a convergence rate $\Oc(1/\sqrt{T})$. This is not surprising since the Lagrangian dual function $G(\lambda)$ is piecewise linear and concave. Thus the negative Lagrangian dual function  is neither smooth nor strongly convex. In general, the convergence rate of gradient-based methods for solving a non-smooth and non-strongly-convex function is at most $\Omega(1/\sqrt{T})$ \cite{bubeck2015convex}. 
It therefore seems impossible to achieve a faster rate, since even with direct access to $\pi^*_\lambda \in \argmin_{\pi} L(\pi, \lambda)$, using the gradient-based PD approach can not have a convergence rate faster than $\Oc(1/\sqrt{T})$ due to the structure of $G(\lambda)$. In this work, however, we show that one can indeed achieve a faster $\Oc(\log(T)/T)$ convergence rate using a novel procedure for updating the dual variable and a correspondingly modified NPG update.

\section{Policy Mirror Descent-Primal Dual (PMD-PD) Algorithm and Main Results}
\label{sec:algoirthm}
In this section, we propose the policy mirror descent-primal dual (PMD-PD) approach (Algorithm \ref{alg:PMD-PD}) for solving the CMDP problem in (\ref{eqn:cmdp}) with an $\Oc(\log(T)/T)$ convergence rate for both the optimality gap and the constraint violation. The PMD-PD algorithm is a two-loop algorithm: the outer loop  updates the  dual variable (Lagrange multiplier) and the  inner loop performs multiple steps of the entropy-regularized NPG updates under the softmax parameterization. Note that while the  standard entropy-regularized NPG algorithm for MDP \cite{cen2021fast} converges only to the optimal policy of the  regularized problem  (which is suboptimal with respect to the unregularized problem), the proposed PMD-PD algorithm converges to the optimal policy of the true (unregularized) CMDP problem. This is achieved by employing entropy regularization with respect to the policy from the previous update as opposed to the uniformly randomized policy used in \cite{cen2021fast}

\paragraph{Outer loop (dual update).}
The PMD-PD algorithm performs the dual variable update in each iteration  of the outer loop (which we call a ``macro" step). The traditional dual update is of the form  $\lambda_{k+1,i} = \min(\lambda_{k,i} + \eta' V_{c_i}^{\pi_{k+1}}(\rho), , 2/((1-\gamma)\xi) )_+$, which requires the knowledge of $\xi$ and fundamentally limits the rate of convergence of PD algorithms due to the lack of smoothness property of the Lagrangian dual function. To overcome this issue, we adopt a modified dual update introduced in \cite{yu2017simple}, where we update the Lagrange multiplier to take a maximum with $-\eta' V_{c_i}^{\pi_{k+1}}(\rho)$ without upper bound $2/((1-\gamma)\xi)$. More precisely, for each $i \in [m]$,  the dual update is given by 
\begin{align}
    \label{eq:dual-update-equation}
     \lambda_{k+1, i} = \max\left\{-\eta' V^{\pi_{k+1}}_{c_i}(\rho), \lambda_{k, i} + \eta' V^{\pi_{k+1}}_{c_i}(\rho) \right\},
\end{align}
where $\eta'$ is the learning rate. We will  show that this modified dual update procedure is helpful in achieving a faster rate of convergence. Below, we state some crucial properties of the dual variables resulting from the above mentioned update procedure. 
\begin{lemma}
\label{lem:L_property}
Let $\lambda_{k} = (\lambda_{k,i}, i \in [m]), k \geq 0,$ be the sequence of dual variables resulting from the PMD-MD algorithm dual update procedure given in Algorithm \ref{alg:PMD-PD}. Then, 
\vspace{-0.3cm}
\begin{enumerate}[itemsep=0pt]
    \item For any macro step k, $\lambda_{k, i} \ge 0, \ \forall i \in [m]$.
    \item For any macro step k, $\lambda_{k, i} + \eta' V_{c_i}^{\pi_k}(\rho) \ge 0, \ \forall i \in [m]$.
    \item For macro step 0, $\|\lambda_{0, i}\|^2 \le \|\eta' V_{c_i}^{\pi_0}(\rho)\|^2$; for any macro step $k > 0$, $\|\lambda_{k, i}\|^2 \ge \|\eta' V_{c_i}^{\pi_k}(\rho)\|^2$, $\forall i \in [m]$.
\end{enumerate}
\end{lemma}
The first property guarantees the feasibility of the Lagrange multipliers; the second property ensures that the Lagrangian in the inner loop can indeed minimize the constraint costs (discussed below); and the third property is a key supporting step for the analysis of the constraint violation.

\paragraph{Inner loop (policy update).}
The PMD-PD algorithm performs $t_{k}$ steps of policy updates in the inner loop for each iteration $k$ of the outer loop (macro step $k$) with a fixed dual variable $\lambda_{k}$. The conventional approach in such a setting is to consider the Lagrangian $L(\pi, \lambda_{k})$ as an MDP problem with the equivalent cost function $c_0(s,a) + \sum_{i=1}^m \lambda_{k,i} c_i(s,a)$ and perform one or multiple gradient updates of this MDP. However, this naive approach does not appear to lead to faster convergence. Different from this conventional approach, we use a modified Lagrange multiplier $\lambda_{k, i} + \eta' V_{c_i}^{\pi_k}(\rho)$ in the inner loop. This is equivalent to considering an MDP with the cost function 
\begin{align}
\label{eqn:tilde-c}
    \tilde{c}_k(s, a) := c_0(s, a) + \sum_{i=1}^m (\lambda_{k, i} + \eta' V_{c_i}^{\pi_k}(\rho)) c_i(s, a). 
\end{align}
It is straightforward to note that $|\tilde{c}_{k}(s, a)| \leq 1 + \sum_{i = 1}^m \lambda_{k, i} + \frac{m \eta'}{1 - \gamma}$. 

We define the state value function and state-action value function of the resulting MDP with cost function $\tilde{c}_k$ as 
\begin{align}
     &\tilde{V}_{k}^{\pi}(s) := \Eb [\sum^{\infty}_{t = 0} \gamma^t \tilde{c}_k(s_t, a_t) ~{\Big |}~ s_0=s,\pi ],
    \label{eqn:tilde-V} \\
    &\hspace{-0.4cm}\tilde{Q}_{k}^{\pi}(s, a)  := \Eb [\sum^{\infty}_{t = 0} \gamma^t \tilde{c}_k(s_t, a_t) ~{\Big |}~ s_0=s, a_0 = a, \pi ].
    \label{eqn:tilde-Q}
\end{align}
It follows that $|\tilde{V}^{\pi}_k(s)| \leq \frac{1 + \sum_{i=1}^m \lambda_{k, i}}{1 - \gamma} + \frac{m \eta'}{(1 - \gamma)^2}$ and $|\tilde{Q}^{\pi}_{k}(s, a)| \leq \frac{1 + \sum_{i=1}^m \lambda_{k, i}}{1 - \gamma} + \frac{m \eta'}{(1 - \gamma)^2}$. We note that $\lambda_{k}$ is upper bounded by a constant that does not depend on $k$ or $K$, see (\ref{eqn:bound_lambda}).

We also define the KL-regularized state value function for coefficient $\alpha > 0$ as 
\begin{align}
    \tilde{V}_{k, \alpha}^{\pi}(s) &:= \Eb [\sum^{\infty}_{t = 0} \gamma^t (\tilde{c}_k(s_t, a_t) + \alpha \log\frac{\pi(a_t|s_t)}{\pi_k(a_t|s_t)}) | s_0 = s, \pi]. \label{eqn:regu-V}
\end{align}
Note that $\tilde{V}^{\pi}_{k, \alpha}(s)$ can be interpreted as a (negative) entropy-regularized value function with cost $\tilde{c}_{k}(s, a) + \alpha\log\frac{1}{\pi_{k}(a| s)}$ \cite{cen2021fast}. The entropy-regularized state-action value function is then defined as \cite{cen2021fast}  
\begin{align}
\label{eqn:regu-Q}
    \tilde{Q}_{k, \alpha}^{\pi}(s, a) = \tilde{c}_k(s, a) &+ \alpha \log\frac{1}{\pi_k(a|s)} + \gamma \Eb_{s' \sim P(\cdot | s, a)} [\tilde{V}_{k, \alpha}^{\pi}(s')]. 
\end{align}
With slight abuse of notation, we denote
\begin{align*}
    \tilde{V}_{k, \alpha}^{\pi}(\rho) := \Eb_{s_0 \sim \rho} [\tilde{V}_{k, \alpha}^{\pi}(s)] \ \text{ and } \ \tilde{V}_{k}^{\pi}(\rho) := \Eb_{s_0 \sim \rho} [\tilde{V}_{k}^{\pi}(s)].
\end{align*}

\begin{algorithm}[t]
\caption{\textbf{Policy Mirror Descent-Primal Dual (PMD-PD)}}
\label{alg:PMD-PD}
\noindent \textbf{Input:} $\rho, K, \alpha, \eta, \eta'$; \\
\noindent \textbf{Initialization:}  $\pi_0$ takes an action uniformly at random for any state. $\lambda_{0, i} = \max\{0, -\eta' V_{c_i}^{\pi_0}(\rho)\}, \forall i \in [m]$;\\
\For{$k = 0, 1, \dots, K-1$}{
\noindent {\bf\textit{[Inner loop (policy update)]}} (for optimizing $\tilde{V}_{k, \alpha}^{\pi}$ given in (\ref{eqn:regu-V}) via regularized NPG) \\
\noindent Initialize $\pi_k^{(0)} = \pi_k$; \\ 
\For{$t = 0, 1, \dots, t_k-1$}{
    \noindent Update the policy $ \pi_k^{(t)}$ to $ \pi_k^{(t+1)}$ according to the regularized NPG update procedure \eqref{eqn:npg-entropy};
}
\noindent $\pi_{k+1}(a|s) =\pi_k^{(t_k)}(a|s), \forall (s, a) \in \Sc \times \Ac$; \\
\noindent{\bf\textit{[Outer loop (dual update)]}} \\
\noindent Update the dual variable $\lambda_{k}$ to $\lambda_{k+1}$ according to the modified dual update procedure \eqref{eq:dual-update-equation};
}
\noindent \textbf{Output:} $\overline{\pi} = \frac{1}{K} \sum_{k=1}^{K} \pi_k$.
\end{algorithm}

The goal of the inner loop is essential to find the optimal policy $\pi^{*}_{k}$ for the entropy-regularized problem, i.e., $\pi^{*}_{k} \in \argmax_{\pi}  \tilde{V}_{k, \alpha}^{\pi}(\rho)$. We achieve this by performing multiple NPG updates. Similar to the NPG update for the unregularized problem given in \eqref{eqn:softmaxNPG}, the NPG update for the regularized problem under the softmax parameterization with the learning rate $\eta$ also yields a closed-form expression \cite{cen2021fast} as given below:
\begin{align}
&\pi_k^{(t+1)}(a|s) = \frac{(\pi_k^{(t)}(a|s))^{1 - \frac{\eta \alpha}{1-\gamma}} \exp(\frac{-\eta \tilde{Q}_{k, \alpha}^{\pi_k^{(t)}}(s, a)}{1-\gamma})}{Z^{(t)}(s)},
\label{eqn:npg-entropy}
\end{align}
where $Z^{(t)}(s) = \sum_{a'} (\pi_k^{(t)}(a'|s))^{1 - \frac{\eta \alpha}{1-\gamma}} \exp(\frac{-\eta \tilde{Q}_{k, \alpha}^{\pi_k^{(t)}}(s, a')}{1-\gamma})$. Indeed, in the inner loop of the PMD-PD algorithm, we update the policy according to the above procedure. 

We summarize the PMD-PD algorithm in Algorithm \ref{alg:PMD-PD}. Note that the name of Algorithm \ref{alg:PMD-PD} comes from the equivalence of NPG and mirror descent under the softmax parameterization, and the reliance on important properties of mirror descent in the analysis. 

We now present the main results on the performance guarantees of the PMD-PD algorithm.  

\begin{theorem}
\label{thm:PMD-MD}
For any $\eta' \in (0, 1]$, let $\alpha = \frac{2 \gamma^2 m \eta'}{(1 - \gamma)^3}$, $\eta = \frac{1-\gamma}{\alpha}$, and take $t_k = \lceil \max(\frac{1}{\eta \alpha} \log(3K C_k), 1) \rceil$ with $C_k = 2 \gamma (\frac{1+\sum_{i=1}^m \lambda_{k,i}}{1-\gamma} + \frac{m \eta'}{(1-\gamma)^2})$. Let $(\pi_{k})_{k \geq 1}$ be the sequence of policies generated by the PMD-PD algorithm (line 9 in Algorithm \ref{alg:PMD-PD}).  Then, for any $K \ge 1$, we have the optimality gap and the constraint violation given by:
\begin{align}
    &\frac{1}{K} \sum_{k=1}^{K} \left(V_{c_0}^{\pi_k}(\rho) - V_{c_0}^{\pi^*}(\rho)\right) \le \frac{1}{K} \left(\frac{\alpha \log(|\Ac|)}{1-\gamma} + 1 + \frac{2}{3(1-\gamma)}\right), \label{eqn:opt-gap} \\
    &\max_{i \in [m]} \left\{\left(\frac{1}{K} \sum_{k=1}^{K} V_{c_{i}}^{\pi_k}(\rho) \right)_{+}\right\} \leq \frac{1}{K} \left(\frac{\|\lambda^*\|}{\eta'} + \sqrt{\frac{\|\lambda^*\|^2}{\eta'^2} + \frac{2\alpha \log(|\Ac|)}{(1 - \gamma)\eta'}  + \frac{2}{\eta'}\left(1 + \frac{2}{3(1-\gamma)}\right)  + \frac{2m}{(1-\gamma)^2}}\right), \label{eqn:con-vio}
\end{align}
where, $\lambda^*$ is the vector of optimal dual variables of the CMDP problem \eqref{eqn:cmdp}. 
\end{theorem}

The corollary below exhibits the $\Oc(\log(T)/T)$ convergence rate in terms of the number of iterations $T$.
\begin{corollary}
\label{cor:PMD-PD}
Denote by $T := \sum_{k=0}^{K-1} t_k$ the total number of iterations of the PMD-PD algorithm (Algorithm \ref{alg:PMD-PD}). The optimality gap and the constraint violation satisfy 
\begin{align*}
    &\frac{1}{K} \sum_{k=1}^{K} \left(V_{c_0}^{\pi_k}(\rho) - V_{c_0}^{\pi^*}(\rho)\right) \le b_1 \frac{m \log|\Ac| \log (C^* T)}{(1-\gamma)^5 T}, \\
    &\max_{i \in [m]} \left\{\left(\frac{1}{K} \sum_{k=1}^{K} V_{c_{i}}^{\pi_k}(\rho)\right)_{+}\right\} \le b_1' \frac{\|\lambda^*\|\sqrt{m \log |\Ac|} \log (C^* T)}{(1-\gamma)^3T},
\end{align*}
where $C^*$ defined in (\ref{eqn:C_*}) is a parameter depending on $\lambda^*$, $b_1 \text{ and } b_1'$ are universal constants.
\end{corollary}

\noindent\textit{Outline of proof idea:} 
We briefly explain the proof idea of Theorem \ref{thm:PMD-MD}, with details presented in Appendix \ref{sec:analysis_PMD}. We analyze the inner loop and outer loop separately. The goal of the inner loop analysis is to show that the policy updates converge to an approximate solution of  the MDP with value $\tilde{V}^{\pi}_{k, \alpha}(\rho)$. We make use of the convergence analysis of the entropy-regularized NPG \cite{cen2021fast} to show this. The analysis of the outer loop focuses on the inner product term $\left\langle \lambda_k + \eta' V^{\pi_{k}}_{c_{1:m}}(\rho),  V^{\pi_{k+1}}_{c_{1:m}}(\rho) \right\rangle$ in the Lagrangian by leveraging the update rule of dual variables in each macro step, which relies on the modified Lagrange multiplier and the modified Lagrangian cost function.

\section{Experiments}
\label{sec:exp}
For the experiments, instead of the  minimization problem \eqref{eqn:cmdp}, we focus on an equivalent maximization problem
\begin{align}
\label{eqn:cmdp_exp}
\max_{\pi}\quad V_{r}^{\pi}(\rho) \quad    \text{s.t.} \quad V_{g_i}^{\pi}(\rho) \geq l_i, \quad \forall i \in [m],
\end{align}
where $r: \Sc \times \Ac \to [0, 1]$ is a reward function and $g_i: \Sc \times \Ac \to [0, 1], \forall i \in [m]$ is a utility function. 
This is mainly to transparently use the existing code base available for policy gradient algorithms. 
The optimality gap and constraint violation are defined as
\begin{align*}
    \text{Optimality Gap}(t) &:= \frac{1}{t} \sum_{\tau=1}^{t} \left(V_{r}^{\pi^*}(\rho) - V_{r}^{\pi_{\tau}}(\rho)\right), \\
    \text{Violation}_i(t) &:= \frac{1}{t} \sum_{\tau=1}^{t} \left(l_i - V_{g_i}^{\pi_{\tau}}(\rho)\right), \ \forall i \in [m].
\end{align*}

We first consider a  randomly generated CMDP with $|\Sc|=20, |\Ac|=10, \gamma=0.8, m=1 \text{ and } l_{1}=3$. We compare the performance of the proposed PMD-PD algorithm with two benchmark algorithms: the NPG-PD algorithm \cite{ding2020natural} and the CRPO algorithm \cite{xu2021crpo}. We choose $\eta=1$ for all algorithms, $\eta'=1$ for NPG-PD and PMD-PD, and $t_k=1, \forall k = 0, 1, \dots, K-1$ for PMD-PD.

Figure \ref{fig:simple_cmdp} illustrates that both the optimality gap and the constraint violation of the PMD-PD algorithm converge faster than those of the NPG-PD algorithm \cite{ding2020natural}. Since the CRPO algorithm \cite{xu2021crpo} focuses on the violated constraint, the updated policy becomes feasible quickly, though at the cost of a slower convergence for the optimality gap.  

As illustrated in Figure \ref{fig:simple_cmdp_log}, the slopes of the NPG-PD algorithm and the CRPO algorithm are around -0.5 in the log-log plot of the optimality gap in Figure \ref{fig:simple_cmdp_log}(a), while the slopes of the PMD-PD algorithm are around -0.9 to -1 in both Figure \ref{fig:simple_cmdp_log}(a) and Figure \ref{fig:simple_cmdp_log}(b), which means that both the optimality gap and the constraint violation of the PMD-PD algorithm converge at a rate of $\tilde{\Oc}(1/t)$.

\vspace{-0.1in}
\paragraph{Additional experiments:} We have included additional experiments in Appendix \ref{sec:exp_appendix}, which show the performance advantages of the sample-based PMD-PD algorithm (see Section \ref{sec:sample}) on the same tabular CMDP, as well as a more complex Acrobot-v1 task from OpenAI Gym \cite{1606.01540}. We have also included the code in the supplementary material.

\begin{figure}[t]
\centering
\begin{subfigure}{0.49\textwidth}
\centering
\includegraphics[width=\textwidth]{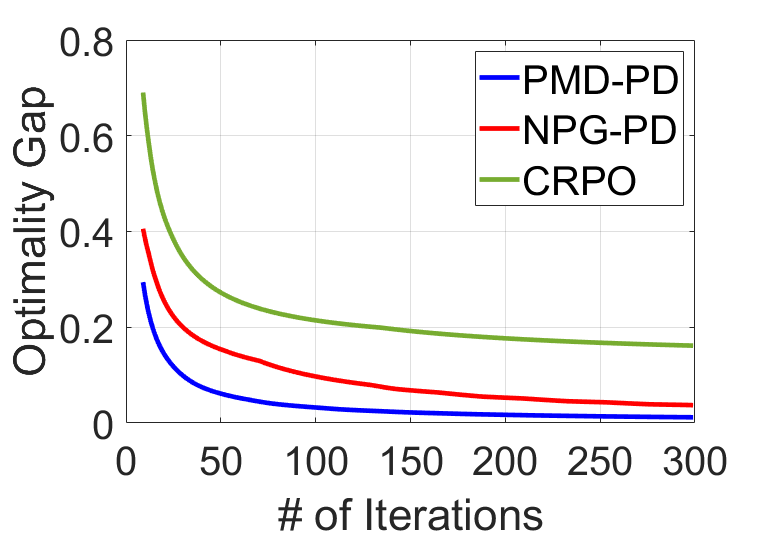}
\caption{}
\end{subfigure}
\begin{subfigure}{0.49\textwidth}
\centering
\includegraphics[width=\textwidth]{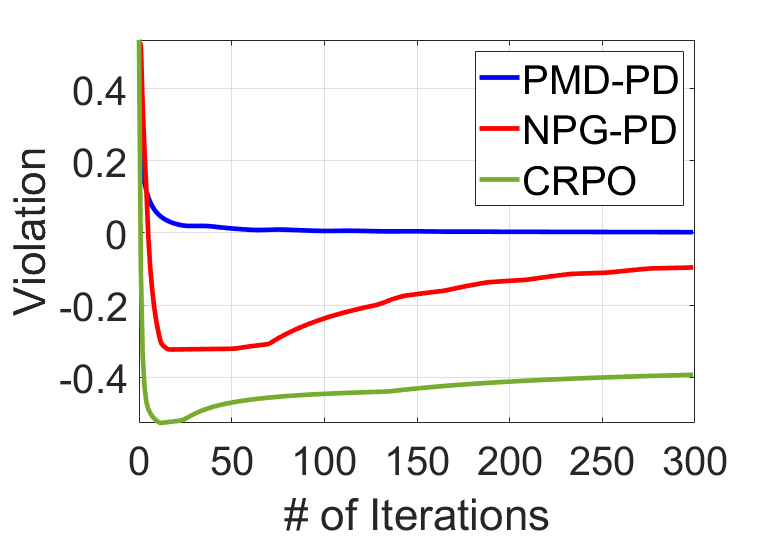}
\caption{}
\end{subfigure}
\caption{The optimality gap and the constraint violation with respect to the number of iterations, for PMD-PD, NPG-PD, and CRPO on a randomly generated CMDP.}
\label{fig:simple_cmdp}
\end{figure}

\begin{figure}[t]
\centering
\begin{subfigure}{0.49\textwidth}
\centering
\includegraphics[width=\textwidth]{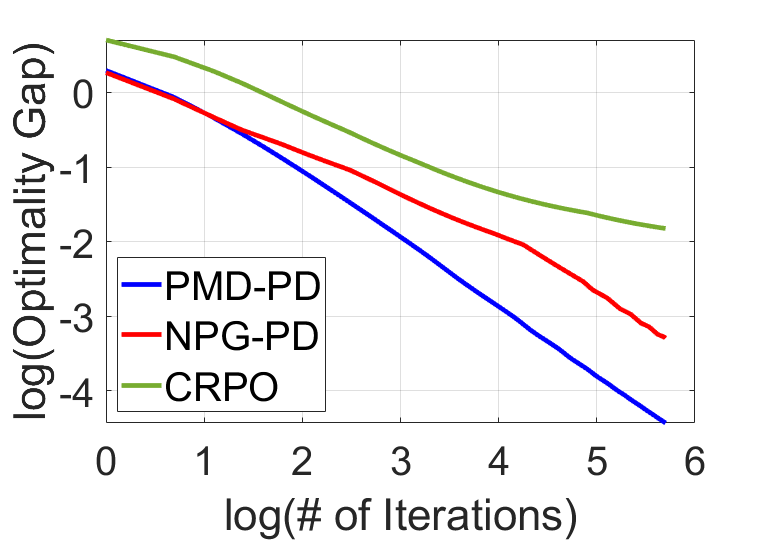}
\caption{}
\end{subfigure}
\begin{subfigure}{0.49\textwidth}
\centering
\includegraphics[width=\textwidth]{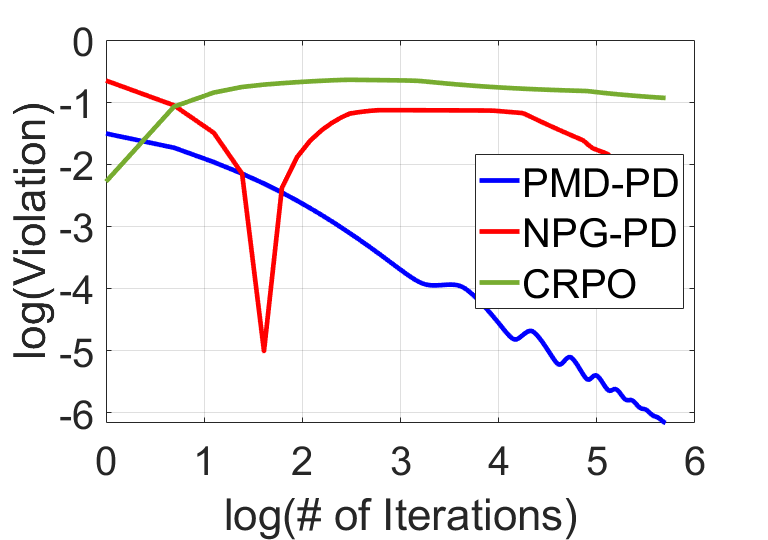}
\caption{}
\end{subfigure}
\caption{The log-log plots of the optimality gap and the constraint violation versus the number of iterations, for PMD-PD, NPG-PD, and CRPO on the same CMDP problem}
\label{fig:simple_cmdp_log}
\end{figure}

\section{Extensions}
\subsection{PMD-PD-Zero: Algorithm with Zero Constraint Violation}
\label{sec:zero}
The CMDP formalism is often used to model control problems with safety constraints \cite{amodei2016concrete, dulac2021challenges, li2016cmdp}. In many of these problems, it is important to ensure that the cumulative constraint violation is zero while finding the optimal policy.  While the PMD-PD algorithm described in the previous section gives provable convergence to the optimal policy, it may incur a positive cumulative constraint violation during the implementation of the algorithm. Indeed, \eqref{eqn:con-vio} in Theorem \ref{thm:PMD-MD} only gives an upper bound on the cumulative constraint violations. One important question in this context is: \textit{Can we design a policy gradient-based algorithm for CMDPs that can provably achieve fast global convergence  while ensuring that the cumulative constraint violation is zero?} 

In recent work, \cite{liu2021learning} used the idea of ``pessimism in the constraints'' to ensure zero constraint violation for a constrained RL problem. We generalize this idea to the policy gradient setting and present a modification of the PMD-PD algorithm, which we call the PMD-PD-Zero Algorithm, to ensure zero cumulative constraint violation. To guarantee this, we require the knowledge of the parameter $\xi$ in Assumption \ref{asm:slater}. 

The key idea of PMD-PD-Zero is to introduce a pessimistic term $\delta \in (0, \xi)$ in the orginal optimization problem \eqref{eqn:cmdp}. More precisely, we consider the pessimistic problem 
\begin{align}
\label{eqn:cmdp_zero}
    \min_{\pi}~~ V_{c_0}^{\pi}(\rho),\quad \text{s.t.}~~ V_{c_i}^{\pi}(\rho) \leq -\delta,\quad \forall i \in [m]. 
\end{align}
By selecting $\delta$ appropriately, and then employing the same update procedure as described in Algorithm \ref{alg:PMD-PD} for \eqref{eqn:cmdp_zero}, we show that we can ensure the same convergence rate for the optimality gap while ensuring zero cumulative constraint violations: 

\begin{theorem}
\label{thm:zero}
Consider Algorithm \ref{alg:PMD-PD} applied to solve the pessimistic CMDP problem (\ref{eqn:cmdp_zero}) under
Assumption \ref{asm:slater}, with the same input parameters specified in Theorem \ref{thm:PMD-MD} and $\delta := \frac{b}{K}$, where $b$ is a parameter specified in (\ref{eqn:c_delta}) depending on $\xi$. When $ K \ge \frac{2b}{\xi}$, the optimality gap and the constraint violation satisfy
\begin{align*}
    &\frac{1}{K} \sum_{k=1}^{K} \left(V_{c_0}^{\pi_k}(\rho) - V_{c_0}^{\pi^*}(\rho)\right) \le \frac{1}{K} \left(\frac{\alpha}{1-\gamma} \log(|\Ac|) + 1 + \frac{2}{3(1-\gamma)}\right) + \frac{2b}{K \xi (1 - \gamma)},\\
    &\max_{i \in [m]} \left\{\left(\frac{1}{K} \sum_{k=1}^{K} V_{c_{i}}^{\pi_k}(\rho)\right)_{+}\right\} = 0.
\end{align*}
\end{theorem}

\begin{corollary}
\label{cor:zero}
Denote by $T = \sum_{k=0}^{K-1} t_k$ the total number of iterations the PMD-PD-Zero algorithm, when $T \ge \frac{2b}{\xi(1-\gamma)} \log(\frac{3 C^* \gamma b}{\xi})$, the optimality gap and the constraint violation satisfy
\begin{align*}
    &\frac{1}{K} \sum_{k=1}^{K} \left(V_{c_0}^{\pi_k}(\rho) - V_{c_0}^{\pi^*}(\rho)\right) \le b_2 \frac{m \log(|\Ac|) \log (C^* T)}{(1-\gamma)^5 T}, \\
    &\max_{i \in [m]} \left\{\left(\frac{1}{K} \sum_{k=1}^{K} V_{c_{i}}^{\pi_k}(\rho) \right)_{+}\right\} = 0,
\end{align*}
where $b_2$ is a universal constant and $C^*$ is a parameter depending on $\lambda^*$ defined in (\ref{eqn:C_*}).
\end{corollary}

\noindent\textit{Outline of proof idea:}
The proof of Theorem \ref{thm:zero} is based on Theorem \ref{thm:PMD-MD}, with details provided in Appendix \ref{sec:zero_appendix}. The key step is to introduce a policy $\pi(\delta)$ that is feasible for the pessimistic CMDP problem (\ref{eqn:cmdp_zero}).

\subsection{Sample-based PMD-PD Algorithm}
\label{sec:sample}
So far we have considered the situation where one has access to an oracle for exact policy evaluation, and established an ${\Oc}(\log T/T)$ global convergence rate for the optimality gap and the constraint violation (Theorem \ref{thm:PMD-MD} and Corollary \ref{cor:PMD-PD}). We now extend Algorithm \ref{alg:PMD-PD} to design a {\it{sample-based}} PMD-PD algorithm without access to the oracle. We assume the existence of a generative model that can generate multiple independent trajectories starting from any arbitrary pair of state and action, as in, for e.g., \cite{lan2021policy, ding2020natural, xu2021crpo}). The performance, i.e., value functions, of a given policy can then be evaluated based on the trajectories.

Compared to Algorithm \ref{alg:PMD-PD}, we only have access to the estimated values $\hat{Q}_{k,\alpha}^{\pi_k^{(t)}}$ and $\hat{V}_{c_i}^{\pi_{k+1}}$ instead of their exact values. Therefore, $\tilde{c}_k(s, a)$ in (\ref{eqn:tilde-c}) and $\lambda_{k+1, i}$ in (\ref{eq:dual-update-equation}) are redefined as 
\begin{align*}
    \tilde{c}_k(s, a) &:= c_0(s, a) + \sum_{i=1}^m (\lambda_{k, i} + \eta' \hat{V}_{c_i}^{\pi_k}(\rho)) c_i(s, a), \\
    \lambda_{k+1, i} &:= \max\{-\eta' \hat{V}_{c_i}^{\pi_{k+1}}(\rho), \lambda_{k, i} + \eta' \hat{V}_{c_i}^{\pi_{k+1}}(\rho)\}.
\end{align*}

For each policy $\pi_k$, we generate $M_{V, k}$ independent trajectories of length $N_{V, k}$ with initial state distribution $\rho$ and compute the estimate 
\begin{align*}
    \hat{V}_{c_i}^{\pi_{k}}(\rho) := \frac{1}{M_{V, k}}\sum_{j = 1}^{M_{V, k}} \sum_{l=0}^{N_{V, k} - 1} \gamma^l c_i(s^j_l, a^j_l),
\end{align*}
where $(s_l^j, a_l^j)$ is the state-action pair at time step $l$ for trajectory $j$. For each policy $\pi_k^{(t)}$, we generate $M_{Q, k}$ independent trajectories of length $N_{Q, k}$ for any state-action pair $(s, a)$, and compute the estimate 
\begin{align*} 
    &\hat{Q}_{k,\alpha}^{\pi_k^{(t)}}(s, a) := \tilde{c}_{k}(s, a) + \alpha \log \frac{1}{\pi_k(a|s)} + \frac{1}{M_{Q, k}} \sum_{j=1}^{M_{Q, k}} \sum_{l=1}^{N_{Q, k}-1} \gamma^l \left[\tilde{c}_k(s_l^j, a_l^j) + \alpha \sum_{a'} \pi_k^{(t)}(a'|s_l^j) \log\frac{\pi_k^{(t)}(a'|s_l^j)}{\pi_k(a'|s_l^j)}\right].
\end{align*}

The detailed description of the algorithm is given in Appendix \ref{sec:sample_appendix} as Algorithm \ref{alg:PMD-PD-A}. The main result of this section is as follows.

\begin{theorem} 
\label{thm:PMD-MD-A}
Fix any confidence parameter $\delta \in (0, 1)$ and precision parameter $\epsilon > 0$, let $\alpha = \frac{2 \gamma^2 m \eta'}{(1 - \gamma)^3}$, $\eta = \frac{1-\gamma}{\alpha}$, $\eta'=1$. Then, with parameters $(K, t_k, M_{V,k}, N_{V,k}, M_{Q, k}, N_{Q,k})$ chosen appropriately \footnote{The detailed description of parameters is provided in Appendix \ref{sec:sample_appendix}.} and with probability at least $1 - \delta$, Algorithm \ref{alg:PMD-PD-A} has the following optimality gap and constraint violation bounds:
\begin{align}
    &\frac{1}{K} \sum_{k=1}^{K} \left(V_{c_0}^{\pi_k}(\rho) - V_{c_0}^{\pi^*}(\rho)\right) \le \epsilon, \label{eqn:sample-gap}\\
    &\max_{i \in [m]} \left\{\left(\frac{1}{K} \sum_{k=1}^{K} V_{c_{i}}^{\pi_k}(\rho) \right)_{+}\right\} \leq \epsilon, \label{eqn:sample-vio}
\end{align}
and the number of queries for the generative model by Algorithm \ref{alg:PMD-PD-A} is $\sum_{k=0}^{K-1} (M_{V, k} N_{V, k} + \sum_{t=0}^{t_k-1} M_{Q, k}N_{Q,k}) = \tilde{\Oc}(1 / \epsilon^3)$. 
\end{theorem}

\noindent\textit{Outline of proof idea:} 
The proof builds on the analysis of Theorem \ref{thm:PMD-MD}, with details presented in Appendix \ref{sec:sample_appendix}. We prove that with probability $1 - \Oc(\delta)$, the estimates are concentrated around the estimated values with precision $\Oc(\epsilon)$, and the dual variables $\{\lambda_k\}_{k=0}^{K-1}$ are uniformly bounded. One may note that the chosen parameters are of order $K = \Theta(1/\epsilon)$, $t_k = \Theta(\log(1/\epsilon))$, $M_{V, k} = \Theta(\log(1/\delta) /\epsilon^2)$, $N_{V,k} = \Theta(\log_{1/\gamma}(1/\epsilon))$, $M_{Q, k} = \Theta( 
\log(1/\delta) / \epsilon^2)$, $N_{Q,k} = \Theta(\log_{1/\gamma}(1/\epsilon))$. It can then be verified that the sample complexity is of order $\tilde{\Oc}(1 / \epsilon^3)$. 

We note that this has an advantage over existing model-free results $\Oc(1/\epsilon^4)$\footnote{The sample complexity in \cite{ding2021provably} is $\Oc(1/\epsilon^2)$, but it needs to estimate transition probabilities $P$ during policy evaluation, which makes it model-based.} \cite{xu2021crpo}.

\section{Conclusion}
We present a new NPG-based algorithm for CMDPs, which enjoys an $\Oc(\log(T)/T)$ global convergence rate for both the optimality gap and the constraint violation employing an oracle for exact policy evaluation. The constraint violations can be reduced to zero by incorporating an additional pessimistic term into the safety constraints, while keeping the same order of convergence rate for the optimality gap. We also extend the oracle-based framework to the sample-based setting enjoying a more efficient sample complexity. A possible future direction for exploration is to develop an algorithm with even faster, possibly with an $\Oc(\exp(-T))$ convergence rate for the optimality gap and the constrained violation.

\section*{Acknowledgement}
P. R. Kumar's work is partially supported by US National Science Foundation under CMMI-2038625, HDR Tripods CCF-1934904; US Office of Naval Research under N00014-21-1-2385; US ARO under W911NF1810331, W911NF2120064; and U.S. Department of Energy's Office of Energy Efficiency and Renewable Energy (EERE) under the Solar Energy Technologies Office Award Number DE-EE0009031. The views expressed herein and conclusions contained in this document are those of the authors and should not be interpreted as representing the views or official policies, either expressed or implied, of the U.S. NSF, ONR, ARO, Department of Energy or the United States Government. The U.S. Government is authorized to reproduce and distribute reprints for Government purposes notwithstanding any copyright notation herein.

Dileep Kalathil gratefully acknowledges funding from the U.S. National Science Foundation (NSF) grants NSF-CRII- CPS-1850206 and NSF-CAREER-EPCN-2045783.

We thank Dongsheng Ding and Tengyu Xu for generously sharing their code in \cite{ding2020natural, xu2021crpo} as baselines.

\bibliographystyle{abbrv}

\newpage
\appendix
\onecolumn

\section{Supporting Definitions and Results}

\subsection{Supporting Preliminaries for Optimization and Estimation}
For the convenience of reading, we collect together some supporting results. 
\begin{definition}[Bregman divergence]
For any convex and differentiable function $h(\cdot)$, the Bregman divergence generated by $h(\cdot)$ is
\begin{align*}
    B_{h}(x, y) := h(x) - h(y) - \langle \nabla h(y), x-y\rangle.
\end{align*}
\end{definition}
An important property associated with Bregman divergences for showing the convergence rates of many first-order algorithms in convex optimization is the ``pushback" property:
\begin{lemma}[Pushback property of Bregman divergences, Lemma 2.1 in \cite{wei2020online}]
\label{lem:pushback}
Let $B_h: \Delta \times \Delta^{o} \to \mathbb{R}$ be a Bregman divergence function, where $\Delta$ is the probability simplex in $\mathbb{R}^{d}$ and $\Delta^{o}$ is the interior of $\Delta$. Let $f: \Delta \rightarrow \mathbb{R}$ be a convex function. Suppose $x^{\star}=\argmin_{x \in \Delta} f(x)+\alpha B_h(x, y)$ for a fixed $y \in \Delta^{o}$ and $\alpha>0$, then, for any $z \in \Delta$,
\begin{align*}
    f\left(x^{\star}\right)+\alpha B_h\left(x^{\star}, y\right) \leq f(z) + \alpha B_h(z, y) - \alpha B_h\left(z, x^{\star}\right).
\end{align*}
\end{lemma}
To analyze the sample-based algorithm, we will use the following standard Hoeffding's inequality.
\begin{lemma}[Hoeffding's inequality \cite{boucheron2013concentration}]
\label{lem:hoeffding}
Let $X_{1}, \ldots, X_{n}$ be independent random variables such that $a \leq X_{i} \leq b$. Let $\bar{X} := \sum_{i=1}^m X_i / m$, then for all $t>0$,
\begin{align*}
    \mathbb{P}\left(\left|\bar{X}-\Eb\left[\bar{X}\right]\right| \geq t\right) \leq 2 \exp \left(-\frac{2 m t^{2}}{(b-a)^2}\right).
\end{align*}
\end{lemma}

\subsection{Supporting Results for Inner Loop of the Proposed Algorithms}
The inner loops of the proposed algorithms optimize an entropy-regularized MDP. The convergence of NPG in entropy-regularized MDP has been well-studied by \cite{cen2021fast}. We first present their key results in the following lemmas. The first one (Lemma \ref{lem:cen}) is for the oracle scenario, while the latter two (Lemma \ref{lem:cen-app} and Lemma \ref{lem:cen_perform}) are applicable to the sample-based case. 
\begin{lemma}[Linear convergence of an exact entropy-regularized NPG, Theorem 1 in \cite{cen2021fast}]
\label{lem:cen}
For any learning rate $0 < \eta \le (1 - \gamma)/\alpha$ and any $k = 0, 1, \dots, K-1$, the entropy-regularized NPG updates satisfy
\begin{align*}
    \left\|\tilde{Q}_{k, \alpha}^{\pi_k^*} - \tilde{Q}_{k, \alpha}^{\pi_k^{(t+1)}}\right\|_{\infty} & \leq C_{k} \gamma(1-\eta \alpha)^{t}, \\
    \left\|\log \pi_{k}^{*}-\log \pi_k^{(t+1)}\right\|_{\infty} & \leq 2 C_{k} \alpha^{-1}(1-\eta \alpha)^{t},\\
    \left\|\tilde{V}_{k, \alpha}^{\pi_k^*} - \tilde{V}_{k, \alpha}^{\pi_k^{(t+1)}}\right\|_{\infty} & \leq 3 C_{k} (1-\eta \alpha)^{t},
\end{align*}
for all $t \geq 0$, where $C_k$ satisfies
\begin{align*}
    C_{k} \geq \left\|\tilde{Q}_{k, \alpha}^{\pi_k^*} - \tilde{Q}_{k, \alpha}^{\pi_k^{(0)}}\right\|_{\infty}+2 \alpha\left(1-\frac{\eta \alpha}{1-\gamma}\right)\left\|\log \pi_{k}^{*}-\log \pi_k^{(0)}\right\|_{\infty}.
\end{align*}
\end{lemma}

\begin{lemma}[Convergence of an approximate entropy-regularized NPG, Theorem 2 in {\cite{cen2021fast}}]
\label{lem:cen-app}
For any learning rate $0 < \eta \le (1 - \gamma)/\alpha$ and any $k = 0, 1, \dots, K-1$, if $\|\hat{Q}_{k, \alpha}^{\pi_k^{(t)}} -  \tilde{Q}_{k, \alpha}^{\pi_k^{(t)}} \|_\infty \leq \delta$ the entropy-regularized NPG updates satisfy
\begin{align*}
    \left\|\tilde{Q}_{k, \alpha}^{\pi_k^*} - \tilde{Q}_{k, \alpha}^{\pi_k^{(t+1)}}\right\|_{\infty} & \leq C_{k} \gamma(1-\eta \alpha)^{t} + \gamma C', \\
    \left\|\log \pi_{k}^{*}-\log \pi_k^{(t+1)}\right\|_{\infty} & \leq 2 C_{k} \alpha^{-1}(1-\eta \alpha)^{t} + 2 \alpha^{-1} C',\\
    \left\|\tilde{V}_{k, \alpha}^{\pi_k^*} - \tilde{V}_{k, \alpha}^{\pi_k^{(t+1)}}\right\|_{\infty} & \leq 3 C_{k} (1-\eta \alpha)^{t} + 3 C',
\end{align*}
for all $t \geq 0$, where $C_k$ and $C'_k$ satisfy
\begin{align*}
    C_{k} \geq \left\|\tilde{Q}_{k, \alpha}^{\pi_k^*} - \tilde{Q}_{k, \alpha}^{\pi_k^{(0)}}\right\|_{\infty}+2 \alpha\left(1-\frac{\eta \alpha}{1-\gamma}\right)\left\|\log \pi_{k}^{*}-\log \pi_k^{(0)}\right\|_{\infty} \quad \text{and} \quad C' & \geq \frac{2 \delta}{ 1 - \gamma} \left(1+\frac{\gamma}{\eta \alpha}\right).
\end{align*}
\end{lemma}
\begin{lemma}[Performance difference of approximate entropy-regularized NPG, Lemma 4 in \cite{cen2021fast}]
\label{lem:cen_perform}
For any learning rate $0 < \eta \le (1 - \gamma)/\alpha$ and any $k = 0, 1, \dots, K-1$,
\begin{align*}
    -V_{k, \alpha}^{\pi_k^{(t)}}(\rho) \leq - V_{k, \alpha}^{\pi_k^{(t+1)}}(\rho) + \frac{2}{1-\gamma} \|\hat{Q}_{k, \alpha}^{\pi_k^{(t)}} - Q_{k, \alpha}^{\pi_k^{(t)}}\|_{\infty}.
\end{align*}
\end{lemma}

The following lemmas characterize the number of iterations required in each inner loop of Algorithms \ref{alg:PMD-PD} and \ref{alg:PMD-PD-A} respectively.
\begin{lemma}[Number of inner-loop iterations for Algorithm \ref{alg:PMD-PD}]
\label{lem:cor_cen}
Let $\eta = (1 - \gamma) / \alpha$. For any $k = 0, 1, \dots, K-1$, if take $t_k = \frac{1}{\eta \alpha} \log(3 C_k K )$ with $C_k = 2 \gamma \left(\frac{1 + \sum_{i=1}^m \lambda_{k, i}}{1-\gamma} + \frac{m \eta'}{(1-\gamma)^2}\right)$ in Algorithm \ref{alg:PMD-PD}, then we have $\tilde{V}^{\pi_{k+1}}_{k, \alpha}(\rho) \leq \tilde{V}^{\pi_k^*}_{k, \alpha}(\rho) + \frac{1}{K}$ and $\| \log \pi^*_{k} - \log\pi_{k+1} \|_{\infty} \leq \frac{2}{3 \alpha K}$. 
\end{lemma}
\begin{proof}[Proof of Lemma \ref{lem:cor_cen}]
According to Lemma \ref{lem:cen}, $\|\tilde{V}_{k, \alpha}^{\pi_k^*} - \tilde{V}_{k, \alpha}^{\pi_{k+1}}\|_{\infty} \le 3 C_k (1 - \eta\alpha)^{t_k}$. 
Since $\log(1 - \eta\alpha) \le - \eta\alpha$, choosing $t_k = \frac{1}{\eta \alpha} \log (3 C_k K)$ gives $3 C_k (1 - \eta\alpha)^{t_k} \le \frac{1}{K}$. It can guarantee $\|\tilde{V}_{k, \alpha}^{\pi_k^*} - \tilde{V}_{k, \alpha}^{\pi_{k+1}}\|_{\infty} \le \frac{1}{K}$ and $\| \log \pi^*_{k} - \log\pi_{k+1} \|_{\infty} \leq \frac{2}{3 \alpha K}$. 
It remains to select $C_k$ appropriately. Notice that $\forall (s, a) \in \Sc \times \Ac$,
\begin{align*}
\left|\tilde{Q}_{k, \alpha}^{\pi_k^*}(s, a) - \tilde{Q}_{k, \alpha}^{\pi_k}(s, a)\right| &= \gamma \sum_{s' \in \Sc} P(s'|s,a) \left|\tilde{V}_{k, \alpha}^{\pi_k^*}(s') - \tilde{V}_{k, \alpha}^{\pi_k}(s')\right| \stackrel{(a)}{\le} \gamma \sum_{s' \in \Sc} P(s'|s,a) \left|\tilde{V}_{k}^{\pi_k^*}(s') - \tilde{V}_{k}^{\pi_k}(s')\right| \\
&\le \gamma \left\|\tilde{V}_{k}^{\pi_k^*} - \tilde{V}_{k}^{\pi_k}\right\|_{\infty} \le 2 \gamma \left(\frac{1 + \sum_{i=1}^m \lambda_{k, i}}{1-\gamma} + \frac{m \eta'}{(1-\gamma)^2}\right),
\end{align*}
where $(a)$ holds due to the relation in (\ref{eqn:inner_relation}). It implies
\begin{align*}
    \left\|\tilde{Q}_{k, \alpha}^{\pi_k^*} - \tilde{Q}_{k, \alpha}^{\pi_k}\right\|_{\infty} \le 2 \gamma \left(\frac{1 + \sum_{i=1}^m \lambda_{k, i}}{1-\gamma} + \frac{m \eta'}{(1-\gamma)^2}\right).
\end{align*}
Since $1-\frac{\eta \alpha}{1-\gamma} = 0$ when $\eta = \frac{1 - \gamma}{\alpha}$, we can therefore conclude the proof by applying the results in Lemma \ref{lem:cen} with $C_k = 2 \gamma \left(\frac{1 + \sum_{i=1}^m \lambda_{k, i}}{1-\gamma} + \frac{m \eta'}{(1-\gamma)^2}\right)$.
\end{proof}

\begin{lemma}[Number of inner-loop iterations for Algorithm \ref{alg:PMD-PD-A}]
\label{lem:cor_cen_estimate}
Let $\eta = (1 - \gamma) / \alpha$. For any $k = 0, 1, \dots, K-1$, if $t_k = \frac{1}{\eta \alpha} \log(3 C_k K)$ with $C_k = 2 \gamma \left(\frac{1 + \sum_{i=1}^m \lambda_{k, i}}{1-\gamma} + \frac{m \eta'}{(1-\gamma)^2}\right)$, and $\|\hat{Q}_{k, \alpha}^{\pi_k^{(t)}} -  \tilde{Q}_{k, \alpha}^{\pi_k^{(t)}} \|_\infty \leq \epsilon$ in Algorithm \ref{alg:PMD-PD-A}, then we have $\tilde{V}^{\pi_{k+1}}_{k, \alpha}(\rho) \leq \tilde{V}^{\pi_k^*}_{k, \alpha}(\rho) + \frac{1}{K} + \frac{6\epsilon}{(1-\gamma)^2}$ and $\| \log \pi^*_{k} - \log\pi_{k+1} \|_{\infty} \leq \frac{2}{3 \alpha K} + \frac{4\epsilon}{\alpha(1-\gamma)^2}$.
\end{lemma}
\begin{proof}[Proof of Lemma \ref{lem:cor_cen_estimate}]
According to Lemma \ref{lem:cen-app} and $\|\hat{Q}_{k, \alpha}^{\pi_k^{(t)}} -  \tilde{Q}_{k, \alpha}^{\pi_k^{(t)}} \|_\infty \leq \epsilon$, we know $\|\tilde{V}_{k, \alpha}^{\pi_k^*} - \tilde{V}_{k, \alpha}^{\pi_{k+1}}\|_{\infty} \le 3 C_k (1 - \eta\alpha)^{t_k} + \frac{6\epsilon}{(1-\gamma)^2}$. Following the same steps as in the proof of Lemma \ref{lem:cor_cen}, we can guarantee $\|\tilde{V}_{k, \alpha}^{\pi_k^*} - \tilde{V}_{k, \alpha}^{\pi_{k+1}}\|_{\infty} \le \frac{1}{K} + \frac{6\epsilon}{(1-\gamma)^2}$ and $\| \log \pi^*_{k} - \log\pi_{k+1} \|_{\infty} \leq \frac{2}{3 \alpha K} + \frac{4\epsilon}{\alpha(1-\gamma)^2}$ by choosing $t_k = \frac{1}{\eta \alpha} \log (3 C_k K)$, where $C_k = 2 \gamma \left(\frac{1 + \sum_{i=1}^m \lambda_{k, i}}{1-\gamma} + \frac{m \eta'}{(1-\gamma)^2}\right)$.
\end{proof}

\subsection{Supporting Results for Outer Loop of the Proposed Algorithm}
\begin{lemma} 
\label{lem:distance}
    Let $d_\rho^{\pi'}, d^{\pi}_\rho$ be two discounted state-action visitation distributions corresponding to policies $\pi'$ and $\pi$. Then
    \begin{align*}
    \|d^{\pi'}_\rho - d^{\pi}_\rho\|_1 \leq \frac{\gamma \sqrt{2}}{1 - \gamma} \sqrt{\min \left(D_{d_\rho^{\pi'}}(\pi' || \pi), D_{d_\rho^{\pi'}}(\pi || \pi'), D_{d_\rho^{\pi}}(\pi' || \pi), D_{d_\rho^{\pi}}(\pi || \pi')\right) }.
\end{align*}
\end{lemma}
\begin{proof}

Let $d^{\pi}_{\rho, h}(\cdot, \cdot)$ be the state-action visitation distribution at step $h$, which implies $\frac{1}{1 - \gamma} d^{\pi}_{\rho} = \sum_{h \geq 0} \gamma^h d^{\pi}_{\rho, h}$. We use $\tilde{\pi}_h := \pi \times h + \pi' \times \infty$ to denote the policy that implements policy $\pi$ for the first $h$ steps and then commits to policy $\pi'$ thereafter. Denote its corresponding discounted state-action visitation distribution by $d^{\tilde{\pi}_h}_\rho$. It follows that
\begin{align*}
    \frac{1}{1 - \gamma}\|d^{\pi'}_\rho - d^{\pi}_\rho\|_1 & \stackrel{(a)}{=} \frac{1}{1 - \gamma} \left\| \sum_{h=0}^\infty (d^{\tilde{\pi}_h}_\rho - d^{\tilde{\pi}_{h+1}}_\rho) \right\|_1 \stackrel{(b)}{\leq} \frac{1}{1 - \gamma} \sum_{h=0}^\infty \|  d^{\tilde{\pi}_h}_\rho - d^{\tilde{\pi}_{h+1}}_\rho \|_1 \\
    & \stackrel{(c)}{\leq} \sum_{h=0}^\infty \sum_{t \geq h+1}^{\infty} \gamma^t \| d^{\tilde{\pi}_h}_{\rho,t} - d^{\tilde{\pi}_{h+1}}_{\rho, t} \|_1 \stackrel{(d)}{\leq} \sum_{h=0}^\infty \sum_{t \geq h+1}^{\infty} \gamma^t \| d^{\tilde{\pi}_h}_{\rho,h} - d^{\tilde{\pi}_{h+1}}_{\rho, h} \|_1 \\
    & = \frac{\gamma}{1 - \gamma} \sum_{h=0}^\infty \gamma^h \Eb_{s \sim d^{\pi}_{\rho, h}}\| \pi(\cdot|s) - \pi'(\cdot|s) \|_1 \\
    & \stackrel{(e)}{\leq} \frac{\gamma}{1 - \gamma} \sqrt{ \left(  \sum_{h \geq 0} \gamma^h \right) \left(\sum_{h = 0}^\infty \gamma^h \Eb_{s \sim d^{\pi}_{\rho, h}}\| \pi(\cdot|s) - \pi'(\cdot|s) \|_1^2 \right)} \\
    & = \frac{\gamma}{(1 - \gamma)^{2}} \sqrt{\Eb_{s \sim d^{\pi}_\rho}\| \pi(\cdot|s) - \pi'(\cdot|s) \|_1^2}.
\end{align*}
Above, $(a)$ holds by telescoping, $(b)$ and $(c)$ hold due to the triangle inequality of $\ell_1$-norm, $(d)$ holds owing to the data processing inequality for $f$-divergence $\|\cdot\|_1$, and $(e)$ holds due to the Cauchy-Schwarz inequality. Due to the symmetry between $\pi$ and $\pi'$, we can similarly derive
\begin{align*}
    \|d^{\pi'}_\rho - d^{\pi}_\rho\|_1 \leq \frac{\gamma}{1 - \gamma} \sqrt{\Eb_{s \sim d^{\pi'}_\rho}\| \pi(\cdot|s) - \pi'(\cdot|s) \|_1^2}.
\end{align*}
We can conclude the proof by further applying Pinsker's inequality.
\end{proof}

\begin{definition}\label{def:pseudo-kl}
Define the pseudo KL-divergence between two discounted state-action visitation distributions $d_{\rho}^{\pi}$ and $d_{\rho}^{\pi'}$ by
\begin{align}
    \tilde{D}(d_{\rho}^{\pi} || d_{\rho}^{\pi'}) : = \sum_{(s, a) \in \Sc \times \Ac} d_{\rho}^{\pi}(s, a) \log \frac{d_{\rho}^{\pi}(s, a) / d_{\rho}^{\pi}(s)}{d_{\rho}^{\pi'}(s, a) / d_{\rho}^{\pi'}(s)}. \label{eqn:tilde-D}
\end{align}
\end{definition}
It is easy to verify that
\begin{align}
\label{eqn:bregman_equ}
    D_{d_{\rho}^{\pi}}(\pi || \pi') & = \sum_{s \in \Sc} d_{\rho}^{\pi}(s) \sum_{a \in \Ac} \pi(a|s) \log \left(\frac{\pi(a|s)}{\pi'(a|s)}\right) = \sum_{s \in \Sc} d_{\rho}^{\pi}(s) \sum_{a \in \Ac} \frac{d_{\rho}^{\pi}(s, a)}{d_{\rho}^{\pi}(s)} \log \left(\frac{d_{\rho}^{\pi}(s, a)/d_{\rho}^{\pi}(s)}{d_{\rho}^{\pi'}(s, a)/d_{\rho}^{\pi'}(s)}\right) \notag \\
    =& \sum_{(s, a) \in \Sc \times \Ac} d_{\rho}^{\pi}(s, a) \log \left(\frac{d_{\rho}^{\pi}(s, a)/d_{\rho}^{\pi}(s)}{d_{\rho}^{\pi'}(s, a)/d_{\rho}^{\pi'}(s)}\right) = \tilde{D}(d_{\rho}^{\pi} || d_{\rho}^{\pi'}).
\end{align}
Though $D_{d_\rho^\pi}(\pi || \pi')$ may not a Bregman divergence with respect to policies, the following lemma shows that $\tilde{D}(d_{\rho}^{\pi} || d_{\rho}^{\pi'})$ is a Bregman divergence of visitation distributions.
\begin{lemma}
\label{lem:Bregman}
    $\tilde{D}(d_{\rho}^{\pi} || d_{\rho}^{\pi'})$ is a Bregman divergence generated by the convex function
\begin{align*}
    \phi(d_{\rho}^{\pi}) = \sum_{(s, a) \in \Sc \times \Ac} d_{\rho}^{\pi}(s, a) \log d_{\rho}^{\pi}(s, a) - \sum_{s \in \Sc} d_{\rho}^{\pi}(s) \log d_{\rho}^{\pi}(s).
\end{align*}
\end{lemma}
\begin{proof}
It is straightforward to verify that
\begin{align*}
    \tilde{D}(d_{\rho}^{\pi} || d_{\rho}^{\pi'}) = \phi(d_{\rho}^{\pi}) - \phi(d_{\rho}^{\pi'}) - \langle \nabla \phi(d_{\rho}^{\pi'}), d_{\rho}^{\pi} - d_{\rho}^{\pi'} \rangle,
\end{align*}
where $\nabla \phi(d_{\rho}^{\pi'})|_{(s, a)} = \log d_{\rho}^{\pi'}(s, a) - \log d_{\rho}^{\pi'}(s)$. Hence we only need to show that $\phi(d_{\rho}^{\pi})$ is convex. The Hessian matrix of function $\phi(d_{\rho}^{\pi})$ can be calculated as $ \diag\left(H_1, H_2, \ldots, H_{|\Sc|}\right)$, where $H_{s} = \frac{1}{d_{\rho}^{\pi}(s)}$ $\left( \diag(d_{\rho}^{\pi}(s)/d_{\rho}^{\pi}(s, \cdot)) - \mathbf{1} \mathbf{1}^T \right)$ is an $|\Ac| \times |\Ac|$ matrix corresponding to state $s$. For each $H_s$, we know for any $x_{1:|\Ac|} \in \Rb^{|\Ac|}$,
\begin{align*}
    x^T H_s x & = \frac{1}{d_{\rho}^{\pi}(s)}\left( \sum_{a \in \Ac} \frac{d_{\rho}^{\pi}(s)}{d_{\rho}^{\pi}(s, a)} x_a^2 - \left(\sum_{a \in \Ac} x_a\right)^2 \right) \\
    &= \frac{1}{d_{\rho}^{\pi}(s)}\left( \left(\sum_{a \in \Ac} \frac{d_{\rho}^{\pi}(s, a)}{d_{\rho}^{\pi}(s)} \right)\left(\sum_{a \in \Ac} \frac{d_{\rho}^{\pi}(s)}{d_{\rho}^{\pi}(s, a)} x_a^2\right) - \left(\sum_{a \in \Ac} x_a\right)^2 \right) \\
    & \stackrel{(a)}{\geq} \frac{1}{d_{\rho}^{\pi}(s)}\left( \left(\sum_{a \in \Ac} |x_a|\right)^2 - \left(\sum_{a \in \Ac} x_a\right)^2 \right) \geq 0,
\end{align*}
where $(a)$ is due to the Cauchy-Schwarz inequality. Thus the Hessian matrix of $\phi(d_{\rho}^{\pi})$ is positive semi-definite, which implies that $\phi(d_{\rho}^{\pi})$ is convex.
\end{proof}

\section{Analysis of the PMD-PD Algorithm (Proof of Theorem \ref{thm:PMD-MD} and Corollary \ref{cor:PMD-PD})}
\label{sec:analysis_PMD}
We illustrate the proofs of the bounds for the optimality gap in (\ref{eqn:opt-gap}) and the constraint violation in (\ref{eqn:con-vio}), with the help of some key supporting lemmas. 

\paragraph{Inner loop analysis.} The goal of the inner loop in macro step $k$ is to approximately solve the MDP with value $\tilde{V}^{\pi}_{k, \alpha}(\rho)$. Let $\pi^*_{k} \in \argmin_{\pi} \tilde{V}^{\pi}_{k, \alpha}(\rho)$ be an optimal policy. We then have 
\begin{align}
\label{eqn:inner_relation}
    \tilde{V}_{k}^{\pi^*_k}(s) \leq \tilde{V}_{k}^{\pi^*_k}(s) + \frac{\alpha}{1 - \gamma}D_{d^{\pi^*}_s}(\pi^* || \pi_k) =\tilde{V}_{k, \alpha}^{\pi^*_k}(s) \leq \tilde{V}^{\pi_k}_{k, \alpha}(s) = \tilde{V}^{\pi_k}_{k}(s),
\end{align}
which implies $|\tilde{V}_{k, \alpha}^{\pi^*_k}(s)|$ and $| \tilde{V}^{\pi_k}_{k, \alpha}(s)|$ are both upper bounded by $\frac{1 + \sum_{i=1}^m \lambda_{k, i}}{1 - \gamma} + \frac{m \eta'}{(1 - \gamma)^2}$. The optimal policy $\pi^*_k$ enjoys the pushback property presented in the following lemma.

\begin{lemma}[Pushback property]
\label{lem:pushback_inner}
For any $k = 0, 1, \ldots, K-1$, and any policy $\pi$
\begin{align*}
    \tilde{V}_k^{\pi_k^*}(\rho) + \frac{\alpha}{1 - \gamma} D_{d^{\pi_k^*}_{\rho}}(\pi_k^* || \pi_k) \leq \tilde{V}_k^{\pi}(\rho) + \frac{\alpha}{1 - \gamma} D_{d^{\pi}_{\rho}}(\pi || \pi_k) - \frac{\alpha}{1 - \gamma} D_{d^{\pi}_{\rho}}(\pi || \pi^*_k).
\end{align*}
\end{lemma}
\begin{proof}
Notice that $\tilde{V}_k^{\pi}(\rho)$ is a convex (in fact a linear) function with respect to the discounted state-action visitation distribution $d^{\pi}_{\rho}$ as shown in (\ref{eqn:LP-CMDP}), and $\tilde{D}(d_\rho^{\pi} || d^{\pi_k}_\rho) = D_{d^{\pi}_\rho}(\pi || \pi_k)$ is a Bregman divergence with respect to $d^{\pi}_{\rho}$ according to (\ref{eqn:bregman_equ}) and Lemma \ref{lem:Bregman}. Recall $\pi_k^* \in  \argmin_{\pi} \tilde{V}_{k, \alpha}^{\pi}(\rho) = \argmin_{\pi} \tilde{V}_k^{\pi}(\rho) + \frac{\alpha}{1 - \gamma} D_{d^{\pi}_{\rho}}(\pi || \pi_k)$. Since policies are under the softmax parameterization, we have $\pi_k(a|s) > 0, \forall (s, a) \in \Sc \times \Ac$, i.e., $\pi_k$ is in the interior of the probability simplex. Choosing $x^* = \pi_k^*, y = \pi_k, z=\pi$ in Lemma \ref{lem:pushback}, we then conclude the proof by the pushback property of the mirror descent for a convex optimization problem.
\end{proof}

The policy $\pi^*_k$ is approximated by $\pi_k^{(t_k)}$ (i.e., $\pi_{k+1}$) via NPG, and it enjoys an almost similar pushback property with an additive approximation error term as in the following lemma.
\begin{lemma}\label{lem:push-app}
Let $\alpha, \eta, t_{0:K-1}$ be the same values as in Theorem \ref{thm:PMD-MD}. Then for any $k = 0, 1, \ldots, K-1$, and any policy $\pi$,
\begin{align*}
    \tilde{V}_k^{\pi_{k+1}}(\rho) + \frac{\alpha}{1 - \gamma} D_{d^{\pi_{k+1}}_{\rho}}(\pi_{k+1} || \pi_k) \leq \tilde{V}_k^{\pi}(\rho) + \frac{\alpha}{1 - \gamma} D_{d^{\pi}_{\rho}}(\pi || \pi_k) - \frac{\alpha}{1 - \gamma} D_{d^{\pi}_{\rho}}(\pi || \pi_{k+1}) + \frac{1}{K}\left(1 + \frac{2}{3(1-\gamma)}\right).
\end{align*}
\end{lemma}
\begin{proof}
According to Lemma \ref{lem:cor_cen}, after $t_k$ iterations of the inner for each macro step $k$, we have\\ $\tilde{V}^{\pi_{k+1}}_{k, \alpha}(\rho) \leq \tilde{V}^{\pi_k^*}_{k, \alpha}(\rho) + \frac{1}{K}$ and $\| \log \pi^*_{k} - \log\pi_{k+1} \|_{\infty} \leq \frac{2}{3 \alpha K}$. It follows that
\begin{align*}
    &\tilde{V}_k^{\pi_{k+1}}(\rho) + \frac{\alpha}{1 - \gamma} D_{d^{\pi_{k+1}}_{\rho}}(\pi_{k+1} || \pi_k) \leq \tilde{V}_k^{\pi_{k}^*}(\rho) + \frac{\alpha}{1 - \gamma} D_{d^{\pi_{k}^*}_{\rho}}(\pi_{k}^* || \pi_k) + \frac{1}{K} \\
    \stackrel{(a)}{\leq}& \tilde{V}_k^{\pi}(\rho) + \frac{\alpha}{1 - \gamma} D_{d^{\pi}_{\rho}}(\pi || \pi_k) - \frac{\alpha}{1 - \gamma} D_{d^{\pi}_{\rho}}(\pi || \pi_{k}^*) + \frac{1}{K}\\
    =& \tilde{V}_k^{\pi}(\rho)  + \frac{\alpha}{1 - \gamma} D_{d^{\pi}_{\rho}}(\pi || \pi_k) - \frac{\alpha}{1 - \gamma} D_{d^{\pi}_{\rho}}(\pi || \pi_{k+1}) + \frac{1}{K} + \frac{\alpha}{1 - \gamma} \sum_{(s, a) \in \Sc \times \Ac} d^{\pi}_{\rho}(s, a) \log \frac{\pi^*_{k}(a|s)}{\pi_{k+1}(a|s)} \\
    \le& \tilde{V}_k^{\pi}(\rho) + \frac{\alpha}{1 - \gamma} D_{d^{\pi}_{\rho}}(\pi || \pi_k) - \frac{\alpha}{1 - \gamma} D_{d^{\pi}_{\rho}}(\pi || \pi_{k+1}) + \frac{1}{K} + \frac{\alpha}{1 - \gamma} \| \log \pi^*_{k} - \log\pi_{k+1} \|_{\infty} \\
    \le& \tilde{V}_k^{\pi}(\rho) + \frac{\alpha}{1 - \gamma} D_{d^{\pi}_{\rho}}(\pi || \pi_k) - \frac{\alpha}{1 - \gamma} D_{d^{\pi}_{\rho}}(\pi || \pi_{k+1}) + \frac{1}{K}\left(1 + \frac{2}{3(1-\gamma)}\right).
\end{align*}
Above, $(a)$ holds due to Lemma \ref{lem:pushback_inner}.
\end{proof}
Comparing Lemma \ref{lem:push-app} with Lemma \ref{lem:pushback_inner}, there is an extra additive term $\frac{1}{K}\left(1 + \frac{2}{3(1-\gamma)}\right)$.

\paragraph{Outer loop analysis} 
The main objective of the analysis of the outer loop is to study the inner product term $\left\langle \lambda_k + \eta' V^{\pi_{k}}_{c_{1:m}}(\rho),  V^{\pi_{k+1}}_{c_{1:m}}(\rho) \right\rangle$ in the Lagrangian by leveraging the update rule of dual variable in each macro step. We first prove Lemma \ref{lem:L_property} about properties of the newly defined Lagrange multipliers.

\begin{lemma}[Restatement of Lemma \ref{lem:L_property}]
\label{lem:L_property_appendix}
Based on the definition of Lagrange multipliers in Algorithm \ref{alg:PMD-PD}, we have
\begin{enumerate}[itemsep=0pt]
    \item For any macro step k, $\lambda_{k, i} \ge 0, \ \forall i \in [m]$.
    \item For any macro step k, $\lambda_{k, i} + \eta' V_{c_i}^{\pi_k}(\rho) \ge 0, \ \forall i \in [m]$.
    \item For macro step 0, $\|\lambda_{0, i}\|^2 \le \|\eta' V_{c_i}^{\pi_0}(\rho)\|^2$; for any macro step $k > 0$, $\|\lambda_{k, i}\|^2 \ge \|\eta' V_{c_i}^{\pi_k}(\rho)\|^2$, $\forall i \in [m]$.
\end{enumerate}
\end{lemma}
\begin{proof}[Proof of Lemma \ref{lem:L_property}]
\begin{enumerate}
    \item Fix $i \in [m]$. Note that $\lambda_{0, i} = \max\{0, -\eta' V_{c_i}^{\pi_0}(\rho)\} \ge 0$ by initialization. Assume $\lambda_{k, i} \ge 0$. If $V_{c_i}^{\pi_{k+1}}(\rho) \ge 0$, then $\lambda_{k+1, i} = \max\{-\eta' V_{c_i}^{\pi_{k+1}}(\rho), \lambda_{k, i} + \eta' V_{c_i}^{\pi_{k+1}}(\rho)\} \ge \lambda_{k, i} + \eta' V_{c_i}^{\pi_{k+1}}(\rho) \ge 0$. If $V_{c_i}^{\pi_{k+1}}(\rho) < 0$, then $\lambda_{k+1, i} = \max\{-\eta' V_{c_i}^{\pi_{k+1}}(\rho), \lambda_{k, i} + \eta' V_{c_i}^{\pi_{k+1}}(\rho)\} \ge -\eta' V_{c_i}^{\pi_{k+1}}(\rho) \ge 0$. Thus, $\lambda_{k+1, i} \ge 0$. The result follows by induction.
    \item Fix $i \in [m]$. Note that by initialization $\lambda_{0, i} + \eta' V_{c_i}^{\pi_0}(\rho) =\max\{0, -\eta' V_{c_i}^{\pi_0}(\rho)\}  + \eta' V_{c_i}^{\pi_0}(\rho) = \max\{0, \eta' V_{c_i}^{\pi_0}(\rho)\} \ge 0$. For $k \ge 0$, we have $\lambda_{k+1, i} = \max\{-\eta' V_{c_i}^{\pi_{k+1}}(\rho), \lambda_{k, i} + \eta' V_{c_i}^{\pi_{k+1}}(\rho)\} \ge -\eta' V_{c_i}^{\pi_{k+1}}(\rho)$.
    \item Fix $i \in [m]$. If $V_{c_i}^{\pi_0}(\rho) \ge 0$, then $\lambda_{0, i} = 0$, thus $\|\lambda_{0, i}\| \le \|\eta' V_{c_i}^{\pi_0}(\rho)\|$. If $V_{c_i}^{\pi_0}(\rho) < 0$, then $\lambda_{0, i} = -\eta' V_{c_i}^{\pi_0}(\rho)$, thus $\|\lambda_{0, i}\| \le \|\eta' V_{c_i}^{\pi_0}(\rho)\|$. For $k \ge 0$, if $V_{c_i}^{\pi_{k+1}}(\rho) \ge 0$, then $\lambda_{k+1, i} = \max\{-\eta' V_{c_i}^{\pi_{k+1}}(\rho), \lambda_{k, i} + \eta' V_{c_i}^{\pi_{k+1}}(\rho)\} \ge \lambda_{k, i} + \eta' V_{c_i}^{\pi_{k+1}}(\rho) \ge \eta' V_{c_i}^{\pi_{k+1}}(\rho)$. If $V_{c_i}^{\pi_{k+1}}(\rho) < 0$, then $\lambda_{k+1, i} = \max\{-\eta' V_{c_i}^{\pi_{k+1}}(\rho), \lambda_{k, i} + \eta' V_{c_i}^{\pi_{k+1}}(\rho)\} \ge -\eta' V_{c_i}^{\pi_{k+1}}(\rho)$. Thus $\|\lambda_{k+1, i}\| \ge \|\eta' V_{c_i}^{\pi_{k+1}}(\rho)\|$.
\end{enumerate}
\end{proof}

\begin{lemma}
\label{lem:lower_bound}
For any $k = 0, 1, \ldots, K-1$,
\begin{align}
    \langle \lambda_{k}, V^{\pi_{k+1}}_{c_{1:m}}(\rho) \rangle \geq \frac{1}{2\eta'} \| \lambda_{k+1} \|^2 - \frac{1}{2\eta'}\| \lambda_{k} \|^2 - \eta' \| V^{\pi_{k+1}}_{c_{1:m}}(\rho) \|^2.
\end{align}
\end{lemma}
\begin{proof}
Recall $\lambda_{k+1, i} = \max\left\{-\eta' V^{\pi_{k+1}}_{c_i}(\rho), \lambda_{k, i} + \eta' V^{\pi_{k+1}}_{c_i}(\rho) \right\}, \forall i \in [m]$.\\
If $\lambda_{k+1, i} = - \eta' V^{\pi_{k+1}}_{c_i}(\rho)$, then 
\begin{align*}
    \frac{1}{2}\lambda_{k+1, i}^2 - \frac{1}{2}\lambda_{k, i}^2 - \eta'^2 (V^{\pi_{k+1}}_{c_i}(\rho))^2 = - \frac{1}{2}\lambda_{k, i}^2 - \frac{\eta'^2}{2} (V^{\pi_{k+1}}_{c_i}(\rho))^2 \leq \eta' \lambda_{k, i} V_{c_i}^{\pi_{k+1}}(\rho),
\end{align*}
which implies $ \lambda_{k, i} V^{\pi_{k+1}}_{c_i}(\rho) \geq \frac{1}{2\eta'} \lambda_{k+1, i}^2 - \frac{1}{2\eta'} \lambda_{k, i}^2 - \eta' (V_{c_i}^{\pi_{k+1}}(\rho))^2.$

If $\lambda_{k+1, i} = \lambda_{k, i} + \eta' V^{\pi_{k+1}}_{c_i}(\rho)$, then
\begin{align*}
    \eta' \lambda_{k, i} V^{\pi_{k+1}}_{c_i}(\rho) &= \frac{1}{2} (\lambda_{k, i} + \eta' V_{c_i}^{\pi_{k+1}}(\rho))^2 - \frac{1}{2} \lambda_{k, i}^2 - \frac{\eta'^2}{2} (V_{c_i}^{\pi_{k+1}}(\rho))^2 \geq \frac{1}{2} \lambda_{k+1, i}^2 - \frac{1}{2} \lambda_{k, i}^2 - \eta'^2 (V_{c_i}^{\pi_{k+1}}(\rho))^2,
\end{align*}
which also implies $ \lambda_{k, i} V^{\pi_{k+1}}_{c_i}(\rho) \geq \frac{1}{2\eta'} \lambda_{k+1, i}^2 - \frac{1}{2\eta'} \lambda_{k, i}^2 - \eta' (V_{c_i}^{\pi_{k+1}}(\rho))^2.$
\end{proof}

\begin{lemma}
\label{lem:inner-product-bound}
For any $k = 0, 1, \ldots, K-1$,
\begin{equation}
\begin{aligned}
    \left\langle \lambda_k + \eta' V^{\pi_{k}}_{c_{1:m}}(\rho),  V^{\pi_{k+1}}_{c_{1:m}}(\rho) \right\rangle \label{eqn:inner-product-bound}
    \geq & \frac{1}{2\eta'} \left(\|\lambda_{k+1}\|^2 - \|\lambda_{k}\|^2\right) + \frac{\eta'}{2} \left(\|V^{\pi_{k}}_{c_{1:m}}(\rho)\|^2 - \|V^{\pi_{k+1}}_{c_{1:m}}(\rho)\|^2\right) \\
    &- \frac{\gamma^2 m \eta'}{(1 - \gamma)^4} D_{d^{\pi_{k+1}}_{\rho}}(\pi_{k+1} || \pi_k). 
\end{aligned}
\end{equation}
\end{lemma}
\begin{proof}
Notice that 
\begin{align}
\label{eqn:decompose_inner}
    \langle \eta' V^{\pi_{k}}_{c_{1:m}}(\rho),  V^{\pi_{k+1}}_{c_{1:m}}(\rho) \rangle = \frac{\eta'}{2} \|V^{\pi_{k}}_{c_{1:m}}(\rho)\|^2 + \frac{\eta'}{2}\|V^{\pi_{k+1}}_{c_{1:m}}(\rho)\|^2 - \frac{\eta'}{2} \|V^{\pi_{k}}_{c_{1:m}}(\rho) - V^{\pi_{k+1}}_{c_{1:m}}(\rho)\|^2.
\end{align}
We bound the last term in the above inequality as below. 

For any $i = 1,2,\ldots, m$, we have
\begin{align*}
   \left|V^{\pi_k}_{c_i}(\rho) - V^{\pi_{k+1}}_{c_i}(\rho)\right|  & = \frac{1}{1-\gamma} \left|\sum_{(s, a) \in \Sc \times \Ac} c_i(s, a) (d_{\rho}^{\pi_k}(s, a) - d_{\rho}^{\pi_{k+1}}(s, a)) \right| \\
    &\leq \frac{1}{1 - \gamma} \|d^{\pi_k}_\rho - d^{\pi_{k+1}}_\rho\|_1 \leq \frac{\gamma \sqrt{2}}{(1 - \gamma)^2} \sqrt{D_{d^{\pi_{k+1}}_{\rho}}(\pi_{k+1} || \pi_k)},
\end{align*}
where the last inequality is due to Lemma \ref{lem:distance}. This implies
\begin{align}
\label{eqn:decompose_inner-2}
    \frac{\eta'}{2} \|V^{\pi_k}_{c_{1:m}}(\rho) - V^{\pi_{k+1}}_{c_{1:m}}(\rho) \|^2 \leq \frac{\gamma^2 m \eta'}{(1 - \gamma)^4} D_{d^{\pi_{k+1}}_{\rho}}(\pi_{k+1} || \pi_k).
\end{align}
Combining \eqref{eqn:decompose_inner} and \eqref{eqn:decompose_inner-2}, we obtain \eqref{eqn:inner-product-bound}.
\end{proof}

\begin{proof}[\textbf{Proof of Theorem \ref{thm:PMD-MD}: Optimality gap bound}]
Here, we give the proof of the optimality gap bound (\ref{eqn:opt-gap}). 

Take $\pi = \pi^*$ in Lemma \ref{lem:push-app}. Since $\lambda_{k, i} + \eta' V^{\pi_k}_{c_i}(\rho) \geq 0$ by the second property in Lemma \ref{lem:L_property}, and $V^{\pi^*}_{c_i}(\rho) \leq 0$ for any $i \in [m]$, we have
\begin{equation}
\label{eqn:inner_analysis}
\begin{aligned}
    &V_{c_0}^{\pi_{k+1}}(\rho) + \left\langle \lambda_k + \eta' V^{\pi_{k}}_{c_{1:m}}(\rho),  V^{\pi_{k+1}}_{c_{1:m}}(\rho) \right\rangle +  \frac{\alpha}{1 - \gamma} D_{d^{\pi_{k+1}}_{\rho}}\left(\pi_{k+1} || \pi_k\right) \\
    \leq & V_{c_0}^{\pi^*}(\rho) + \frac{\alpha}{1 - \gamma} D_{d^{\pi^*}_{\rho}}(\pi^* || \pi_k) - \frac{\alpha}{1 - \gamma} D_{d^{\pi^*}_{\rho}}(\pi^* || \pi_{k+1}) + \frac{1}{K}\left(1 + \frac{2}{3(1-\gamma)}\right),
\end{aligned}
\end{equation}
where we use the shorthand $V_{c_{1:m}}^{\pi}(\rho) := \left(V_{c_1}^{\pi}(\rho), \dots, V_{c_m}^{\pi}(\rho)\right)$.

Substituting the lower bound of inner product $\langle \lambda_k + \eta' V^{\pi_{k}}_{c_{1:m}}(\rho),  V^{\pi_{k+1}}_{c_{1:m}}(\rho) \rangle$ in (\ref{eqn:inner-product-bound}) from Lemma \ref{lem:inner-product-bound} into (\ref{eqn:inner_analysis}) leads to
\begin{align*}
    & V_{c_0}^{\pi_{k+1}}(\rho) + \frac{1}{2\eta'} \left(\|\lambda_{k+1}\|^2 - \|\lambda_{k}\|^2\right) + \frac{\eta'}{2}\left(\|V^{\pi_{k}}_{c_{1:m}}(\rho)\|^2 - \|V^{\pi_{k+1}}_{c_{1:m}}(\rho)\|^2\right) + \frac{\alpha (1-\gamma)^3 - \gamma^2 m \eta'}{(1 - \gamma)^4} D_{d^{\pi_{k+1}}_{\rho}}(\pi_{k+1} || \pi_k) \\
    \leq& V_{c_0}^{\pi^*}(\rho) + \frac{\alpha}{1 - \gamma} D_{d^{\pi^*}_{\rho}}(\pi^* || \pi_k) - \frac{\alpha}{1 - \gamma} D_{d^{\pi^*}_{\rho}}(\pi^* || \pi_{k+1}) + \frac{1}{K}\left(1 + \frac{2}{3(1-\gamma)}\right).
\end{align*}
When $\alpha = \frac{2 \gamma^2 m \eta'}{(1 - \gamma)^3}$, $\left( \frac{\alpha}{1 - \gamma} - \frac{\gamma^2 m \eta'}{(1 - \gamma)^4} \right) D_{d^{\pi_{k+1}}_{\rho}}(\pi_{k+1} || \pi_k) \geq 0$, and it follows from telescoping that
\begin{align}
    \sum_{k = 1}^{K} V^{\pi_k}_{c_0}(\rho) & \leq K V^{\pi^*}_{c_0}(\rho) + \frac{\alpha}{1 - \gamma} D_{d^{\pi^*}_\rho}(\pi^* || \pi_{0}) - \frac{\alpha}{1 - \gamma} D_{d^{\pi^*}_\rho}(\pi^* || \pi_{K}) + 1 + \frac{2}{3(1-\gamma)} \notag \\
    &\quad + \frac{\eta'}{2} \|V^{\pi_{K}}_{c_{1:m}}(\rho)\|^2 - \frac{\eta'}{2} \|V^{\pi_{0}}_{c_{1:m}}(\rho)\|^2 + \frac{1}{2\eta'} \|\lambda_0\|^2 - \frac{1}{2\eta'} \|\lambda_K\|^2 \label{eqn:thm_first_init} \\
    & \stackrel{(a)}{\leq} K V^{\pi^*}_{c_0}(\rho) + \frac{\alpha}{1 - \gamma} D_{d^{\pi^*}_\rho}(\pi^* || \pi_{0}) - \frac{\alpha}{1 - \gamma} D_{d^{\pi^*}_\rho}(\pi^* || \pi_{K}) + 1 + \frac{2}{3(1-\gamma)} \notag \\
    & \stackrel{(b)}{\leq} K V_{c_0}^{\pi^*}(\rho) + \frac{\alpha}{1 - \gamma} \log(|\mathcal{A}|) + 1 + \frac{2}{3(1-\gamma)}.
\label{eqn:thm_first}
\end{align}
$(a)$ holds due to the third property of Lemma \ref{lem:L_property} and $(b)$ holds since $\pi_0$ is the uniformly distributed policy and thus $D_{d^{\pi^*}_\rho}(\pi^* || \pi_{0}) = \sum_{s \in \Sc} d_{\rho}^{\pi^*}(s) \sum_{a \in \Ac}$ $\pi^*(a|s) \log (|\Ac| \pi^*(a|s)) \le \log(|\Ac|)$. We now get the bound (\ref{eqn:opt-gap}) by dividing by $K$ on both sides.
\end{proof}

\begin{proof}[\textbf{Proof of Theorem \ref{thm:PMD-MD}: Constraint violation bound}]
Here, we give the proof of the constraint violation bound (\ref{eqn:con-vio}). 

For any $i \in [m]$, since $\lambda_{k, i} = \max\{-\eta' V_{c_i}^{\pi_{k}}(\rho), \lambda_{k-1, i} + \eta' V_{c_i}^{\pi_{k}}(\rho)\} \ge \lambda_{k-1, i} + \eta' V_{c_i}^{\pi_{k}}(\rho)$, we have
\begin{align}
\label{eq:cvb-st-1}
    \sum_{k = 1}^{K} V^{\pi_k}_{c_i}(\rho) \leq \frac{\lambda_{K, i} - \lambda_{0, i}}{\eta'} \leq \frac{\lambda_{K, i}}{\eta'} \leq \frac{\|\lambda_K\|}{\eta'}.
\end{align}

To analyze the constraint violation, it therefore suffices to bound the dual variables. Consider the Lagrangian with optimal dual variable $L(\pi, \lambda^*) = V^{\pi}_{c_0}(\rho) + \sum_{i=1}^m \lambda^*_i V^{\pi}_{c_i}(\rho)$, whose minimum value $V^{\pi^*}_{c_0}(\rho)$ is achieved by the optimal policy $\pi^*$. We know
\begin{align*}
    & K V^{\pi^*}_{c_0}(\rho) \stackrel{(a)}{=} K L(\pi^*, \lambda^*) \leq \sum_{k = 1}^{K} L(\pi_k, \lambda^*) = \sum_{k=1}^K V^{\pi_k}_{c_0}(\rho) + \sum_{i = 1}^m \lambda^*_i \sum_{k=1}^K V^{\pi_k}_{c_i}(\rho)  \stackrel{(b)}{\leq} \sum_{k=1}^K V^{\pi_k}_{c_0}(\rho) + \frac{1}{\eta'} \sum_{i = 1}^m \lambda^*_i \lambda_{K, i} \\
    \stackrel{(c)}{\leq}& K V^{\pi^*}_{c_0}(\rho) + \frac{\alpha}{1 - \gamma} \left(D_{d^{\pi^*}_\rho}(\pi^* || \pi_{0}) - D_{d^{\pi^*}_\rho}(\pi^* || \pi_{K})\right) + 1 + \frac{2}{3(1-\gamma)} + \frac{\eta'}{2} \|V^{\pi_{K}}_{c_{1:m}}(\rho)\|^2 - \frac{1}{2\eta'} \|\lambda_K\|^2 + \frac{1}{\eta'} \sum_{i = 1}^m \lambda^*_i \lambda_{K, i}.
\end{align*}
$(a)$ holds due to the complementary slackness, $(b)$ follows from \eqref{eq:cvb-st-1}, and  $(c)$ follows from (\ref{eqn:thm_first_init}) and the third property of Lemma \ref{lem:L_property}. This implies
\begin{align}
    & \frac{1}{2\eta'} \|\lambda_{K}\|^2 - \frac{1}{\eta'} \sum_{i = 1}^m \lambda^*_i \lambda_{K, i} \leq \frac{\alpha}{1 - \gamma} \left(D_{d^{\pi^*}_\rho}(\pi^* || \pi_{0}) - D_{d^{\pi^*}_\rho}(\pi^* || \pi_{K})\right) + 1 + \frac{2}{3(1-\gamma)} + \frac{\eta'}{2} \|V^{\pi_{K}}_{c_{1:m}}(\rho)\|^2 \notag \\
    \stackrel{(d)}{\leq}& \frac{\alpha}{1 - \gamma} \log(|\Ac|) - \frac{(1-\gamma)^3 \alpha}{2 \gamma^2 m} \|V^{\pi_{K}}_{c_{1:m}}(\rho) - V^{\pi^*}_{c_{1:m}}(\rho)\|^2 + 1 + \frac{2}{3(1-\gamma)} + \frac{\eta'}{2} \left\|\left(V^{\pi_{K}}_{c_{1:m}}(\rho) - V^{\pi^*}_{c_{1:m}}(\rho)\right) + V^{\pi^*}_{c_{1:m}}(\rho)\right\|^2 \notag \\
    \stackrel{(e)}{=} & \frac{\alpha}{1 - \gamma} \log(|\Ac|) + 1 + \frac{2}{3(1-\gamma)} + \left(\frac{\eta'}{2} - \frac{\gamma^2 m \eta'^2}{2[\gamma^2 m \eta' - (1 - \gamma)^3 \alpha]}\right) \|V^{\pi^*}_{c_{1:m}}(\rho)\|^2 \notag \\
    & + \frac{\gamma^2 m \eta' - (1-\gamma)^3 \alpha}{2 \gamma^2 m} \left\|V^{\pi_{K}}_{c_{1:m}}(\rho) - V^{\pi^*}_{c_{1:m}}(\rho) + \frac{\gamma^2 m \eta'}{\gamma^2 m \eta' - (1-\gamma)^3 \alpha} V^{\pi^*}_{c_{1:m}}(\rho)\right\|^2,  \label{eqn:complex-equ}
\end{align}
where $(d)$ follows by using the lower bound for $D_{d^{\pi^*}_\rho}(\pi^* || \pi_{K})$ from  \eqref{eqn:decompose_inner-2} (by substituting  $\pi = \pi_K$, $\pi' = \pi^*$), and upper bounding  $D_{d^{\pi^*}_\rho}(\pi^* || \pi_{0}) \le \log(|\Ac|)$. We obtain $(e)$ by the fact that 
\begin{align*}
    -a \|x\|^2 + b \|x + y\|^2 = (b - \frac{b^2}{b-a})\|y\|^2 + (b-a) \|x + \frac{b}{b-a}y\|^2 ,
\end{align*}
and substituting $a = \frac{(1-\gamma)^3 \alpha}{2 \gamma^2 m}, b=\frac{\eta'}{2}, x = V^{\pi_K}_{c_{1:m}}(\rho) - V^{\pi^*}_{c_{1:m}}(\rho), y = V^{\pi^*}_{c_{1:m}}(\rho)$ into the above equation. When $\alpha = \frac{2 \gamma^2 m \eta'}{(1 - \gamma)^3}$, $\frac{\gamma^2 m \eta' - (1-\gamma)^3 \alpha}{2 \gamma^2 m} \leq 0$ and $\frac{\eta'}{2} - \frac{\gamma^2 m \eta'^2}{2[\gamma^2 m \eta' - (1 - \gamma)^3 \alpha]} = \eta'$. It then follows that
\begin{align}
    \frac{1}{2\eta'}\|\lambda^* - \lambda_K\|^2 =& \frac{1}{2\eta'}\|\lambda^*\|^2 + \frac{1}{2\eta'} \|\lambda_{K}\|^2 - \frac{1}{\eta'} \sum_{i = 1}^m \lambda^*_i \lambda_{K, i} \notag\\
    \leq& \frac{1}{2\eta'} \|\lambda^*\|^2 + \frac{\alpha}{1 - \gamma} \log(|\Ac|) + 1 + \frac{2}{3(1-\gamma)} + \eta' \|V^{\pi^*}_{c_{1:m}}(\rho) \|^2 \notag\\
    \leq & \frac{1}{2\eta'} \|\lambda^*\|^2 + \frac{\alpha}{1 - \gamma} \log(|\Ac|) + 1 + \frac{2}{3(1-\gamma)} + \frac{m \eta'}{(1-\gamma)^2}. \label{eqn:lambda_bound}
\end{align}
Using the above bound in \eqref{eq:cvb-st-1}, we get 
\begin{align}
    \sum_{k = 1}^{K} V^{\pi_k}_{c_i}(\rho) \leq \frac{\|\lambda^*\|}{\eta'} + \frac{\|\lambda_{K} - \lambda^*\|}{\eta'} \leq \frac{\|\lambda^*\|}{\eta'} + \sqrt{ \frac{\|\lambda^*\|^2}{\eta'^2} + \frac{2\alpha}{(1 - \gamma)\eta'} \log(|\Ac|) + \frac{2}{\eta'}(1 + \frac{2}{3(1-\gamma)}) + \frac{2m}{(1-\gamma)^2}},
\label{eqn:thm_second}
\end{align}
from which we obtain the  constraint violation upper bound given in (\ref{eqn:con-vio}).
\end{proof}

\begin{proof}[\textbf{Proof of Corollary \ref{cor:PMD-PD}}]
According to (\ref{eqn:lambda_bound}), $\forall k \geq 1$,
\begin{align}
\label{eqn:bound_lambda}
    \|\lambda_k\| \leq \|\lambda^*\| + \sqrt{\|\lambda^*\|^2 + 2 \eta'\left(\frac{\alpha}{1-\gamma} \log(|\Ac|) + 1 + \frac{2}{3(1-\gamma)} + \frac{m \eta'}{(1-\gamma)^2}\right)}.
\end{align}
Recall $T = \sum_{k=0}^{K-1} t_k \leq \frac{1}{1-\gamma} \sum_{k=0}^{K-1} (\log(3 K C_k) + 1)$, where 
\begin{align}
\label{eqn:C_*}
    & 3e C_k = 6e \gamma \left(\frac{1 + \sum_{i=1}^m \lambda_{k, i}}{1-\gamma} + \frac{m \eta'}{(1-\gamma)^2}\right) \leq 6e \gamma \left(\frac{1 + \sqrt{m} \|\lambda_k\|}{1-\gamma} + \frac{m \eta'}{(1-\gamma)^2}\right) \notag \\
    \le& 6e \gamma \left(\frac{1 + \sqrt{m} \left(\|\lambda^*\| + \sqrt{\|\lambda^*\|^2 + 2 \eta'\left(\frac{\alpha}{1-\gamma} \log(|\Ac|) + 1 + \frac{2}{3(1-\gamma)} + \frac{m \eta'}{(1-\gamma)^2}\right)}\right)}{1-\gamma} + \frac{m \eta'}{(1-\gamma)^2}\right) =: C^*,
\end{align}
where $C^*$ does not depend on $K$. 

Therefore, we have $\frac{K}{1-\gamma} \leq T \leq \frac{1}{1-\gamma} K \log(K) + \frac{\log(C^*)}{1 - \gamma}K$, which further implies $\frac{1}{K} \leq \frac{\log(C^* T)}{(1-\gamma)T}$. Now, from Theorem \ref{thm:PMD-MD},
\begin{align*}
    \frac{1}{K} \sum_{k=1}^{K} \left(V_{c_0}^{\pi_k}(\rho) - V_{c_0}^{\pi^*}(\rho)\right) &\le \frac{1}{K} \left(\frac{\alpha \log(|\Ac|)}{1-\gamma} + 1 + \frac{2}{3(1-\gamma)}\right) \le b_1 \frac{m \log(|\Ac|) \log (C^* T)}{(1-\gamma)^5 T}, \\
    \max_{i \in [m]} \left\{\left(\frac{1}{K} \sum_{k=1}^{K} V_{c_{i}}^{\pi_k}(\rho) \right)_{+}\right\} &\leq \frac{1}{K} \left(\frac{\|\lambda^*\|_{2}}{\eta'} + \sqrt{\frac{\|\lambda^*\|^2_{2}}{\eta'^2} + \frac{2\alpha \log(|\Ac|)}{(1 - \gamma)\eta'}  + \frac{2}{\eta'}\left(1 + \frac{2}{3(1-\gamma)}\right) + \frac{2m}{(1-\gamma)^2}}\right) \\
    &\le b_1' \frac{\sqrt{m \log(|\Ac|) \|\lambda^*\|^2} \log (C^* T)}{(1-\gamma)^3T},
\end{align*}
where $b_1$ and $b_1'$ are universal constants.
\end{proof}

\section{Analysis of the PMD-PD-Zero Algorithm (Proof of Theorem \ref{thm:zero})}
\label{sec:zero_appendix}
\begin{lemma}
\label{lem:upper-dual}
Under Assumption \ref{asm:slater}, the optimal dual variables $\lambda^*$ satisfies
\begin{align*}
    \|\lambda^*\| \leq \|\lambda^*\|_1 \leq \frac{2}{\xi(1 - \gamma)}.
\end{align*}
\end{lemma}
\begin{proof}
Let $\pi^*$ and $\lambda^*$ achieve the minimax solution of the Lagrangian $L(\pi, \lambda)$. If $V^{\pi^*}_{c_i}(\rho) < 0$ for some $i \in [m]$, it follows that $\lambda^*_i = 0$. 
Due to Assumption \ref{asm:slater},  
\begin{align*}
    V^{\pi^*}_{c_0}(\rho) = V^{\pi^*}_{c_0}(\rho) + \sum_{i = 1}^m \lambda^*_i V^{\pi^*}_{c_i}(\rho) \leq V^{\bar{\pi}}_{c_0}(\rho) + \sum_{i = 1}^m \lambda^*_i V^{\bar{\pi}}_{c_i}(\rho) \leq V^{\bar{\pi}}_{c_0}(\rho) - \xi \sum_{i = 1}^m \lambda^*_i,
\end{align*}
which implies that $ \xi \|\lambda^*\| \leq \xi \|\lambda^*\|_1 = \xi \sum_{i = 1}^m \lambda^*_i \leq V^{\bar{\pi}}_{c_0}(\rho) - V^{\pi^*}_{c_0}(\rho)$. 

Hence, $ \|\lambda^*\| \leq \|\lambda^*\|_1 \leq \frac{V^{\bar{\pi}}_{c_0}(\rho) - V^{\pi^*}_{c_0}(\rho)}{\xi} \leq \frac{2}{\xi(1 - \gamma)}.$
\end{proof}

\begin{theorem}[Restatement of Theorem \ref{thm:zero}]
\label{thm:zero_appendix}
Suppose Assumption 2 holds. Consider Algorithm \ref{alg:PMD-PD}  applied to solve the pessimistic CMDP problem (\ref{eqn:cmdp_zero}) with any $\eta' \in (0, 1]$, $\alpha = \frac{2 \gamma^2 m \eta'}{(1 - \gamma)^3}$, $\eta = \frac{1-\gamma}{\alpha}$, $t_k = \frac{1}{\eta \alpha} \log(3K C_k \gamma)$ with $C_k = 2 \gamma \left(\frac{1+\sum_{i=1}^m \lambda_{k,i}}{1-\gamma} + \frac{m \eta'}{(1-\gamma)^2}\right)$, $\delta = \frac{b}{K}$, where  
\begin{align}
\label{eqn:c_delta}
b := \left(\frac{4}{\xi(1-\gamma)\eta'} + \sqrt{\frac{16}{\xi^2 (1-\gamma)^2 \eta'^2} + \frac{2\alpha}{(1-\gamma)\eta'} \log(|\Ac|) + \frac{2}{\eta'} (1 + \frac{2}{3(1-\gamma)}) + \frac{2m}{(1-\gamma)^2}}\right).
\end{align}
Then, $\forall K \ge \frac{2b}{\xi}$, we have the optimality gap and the constraint violation bounded as follows:
\begin{align*}
    &\frac{1}{K} \sum_{k=1}^{K} \left(V_{c_0}^{\pi_k}(\rho) - V_{c_0}^{\pi^*}(\rho)\right) \le \left(\frac{\alpha}{1-\gamma} \log(|\Ac|) + 1 + \frac{2}{3(1-\gamma)}\right) \frac{1}{K} + \frac{2\delta}{\xi (1 - \gamma)}\\
    &\max_{i \in [m]} \left\{\left(\frac{1}{K} \sum_{k=1}^{K} V_{c_{1:m}}^{\pi_k}(\rho)\right)_{+} \right\} = 0.
\end{align*}
\end{theorem}
\begin{proof}[\textbf{Proof of Theorem \ref{thm:zero}}]
Since $\pi(a|s) = d^{\pi}(s, a)/\sum_{a' \in \Ac} d^{\pi}(s, a')$, we can define a mixed state-action visitation distribution $d^{\pi(\delta)}$ as
\begin{align*}
    d^{\pi(\delta)}(s, a) = \frac{\xi-\delta}{\xi}d^{\pi^*}(s, a) + \frac{\delta}{\xi} d^{\overline{\pi}}(s, a), \ \forall (s, a) \in \Sc \times \Ac.
\end{align*}
It is easy to verify that $\pi(\delta)$ is a feasible solution to the new CMDP formulation (\ref{eqn:cmdp_zero}) since $\forall i \in [m]$,
\begin{align*}
    V_{c_i}^{\pi(\delta)}(\rho) &= \langle c_i, d_{\rho}^{\pi(\delta)} \rangle = \frac{\xi - \delta}{\xi} V_{c_i}^{\pi^*}(\rho) + \frac{\delta}{\xi} V_{c_i}^{\overline{\pi}(\delta)}(\rho) \le 0 + \frac{\delta}{\xi} (-\xi) = -\delta.
\end{align*}
Let $\pi^*(\delta)$ be the optimal policy of the new CMDP problem (\ref{eqn:cmdp_zero}). It implies
\begin{align}
\label{eqn:zero_additional}
    V_{c_0}^{\pi^*(\delta)}(\rho) - V_{c_0}^{\pi^*}(\rho) \le V_{c_0}^{\pi(\delta)}(\rho) - V_{c_0}^{\pi^*}(\rho) \le \frac{\delta}{\xi} (V_{c_0}^{\overline{\pi}}(\rho) - V_{c_0}^{\pi^*}(\rho)) \le \frac{2 \delta}{\xi(1-\gamma)}.
\end{align}
Therefore,
\begin{align*}
    \frac{1}{K} \sum_{k=1}^{K} \left(V_{c_0}^{\pi_k}(\rho) - V_{c_0}^{\pi^*}(\rho)\right) &= \frac{1}{K} \sum_{k=1}^{K} \left[\left(V_{c_0}^{\pi_k}(\rho) - V_{c_0}^{\pi^*(\delta)}(\rho)\right) + \left(V_{c_0}^{\pi^*(\delta)}(\rho) - V_{c_0}^{\pi^*}(\rho)\right)\right] \\
    &\stackrel{(a)}{\le} \left(\frac{\alpha}{1-\gamma} \log(|\Ac|) + 1 + \frac{2}{3(1-\gamma)}\right) \frac{1}{K} + \frac{2 \delta}{\xi (1 - \gamma)}
\end{align*}
$(a)$ holds due to  (\ref{eqn:opt-gap}) from Theorem \ref{thm:PMD-MD} and (\ref{eqn:zero_additional}).

Let $\|\lambda_{\delta}^*\|$ be the optimal dual variable for the pessimistic CMDP problem (\ref{eqn:cmdp_zero}). Now,
\begin{align*}
    \left(\frac{1}{K} \sum_{k=1}^{K} V_{c_{i}}^{\pi_k}(\rho)\right)_{+} &= \left(\frac{1}{K} \sum_{k=1}^{K} \left(V_{c_{i}}^{\pi_k}(\rho) + \delta\right) - \delta\right)_{+} \\
    &\stackrel{(b)}{\leq} \left(\left(\frac{\|\lambda_{\delta}^*\|}{\eta'} + \sqrt{\frac{\|\lambda_{\delta}^*\|^2}{\eta'^2} + \frac{2\alpha}{(1 - \gamma)\eta'} \log(|\Ac|) + \frac{2}{\eta'}(1 + \frac{2}{3(1-\gamma)}) + \frac{2m}{(1-\gamma)^2}}\right)\frac{1}{K} - \delta \right)_{+},
\end{align*}
where $(b)$ holds due to (\ref{eqn:con-vio}) from Theorem \ref{thm:PMD-MD}. According to Lemma \ref{lem:upper-dual},  
\begin{align*}
    \|\lambda_{\delta}^*\| \le \frac{2}{(\xi - \delta)(1-\gamma)}. 
\end{align*}

When $b / K \le \xi/2$, i.e., $K \ge \frac{2b}{\xi}$, choosing
\begin{align*}
    \delta = \left(\frac{4}{\xi(1-\gamma)\eta'} + \sqrt{\frac{16}{\xi^2 (1-\gamma)^2 \eta'^2} + \frac{2\alpha}{(1-\gamma)\eta'} \log(|\Ac|) + \frac{2}{\eta'} (1 + \frac{2}{3(1-\gamma)}) + \frac{2m}{(1-\gamma)^2}}\right)\frac{1}{K}
\end{align*}
concludes the proof. 
\end{proof}

\section{Analysis of the  PMD-PD Algorithm with Sample-based Approximation (Proof of Theorem \ref{thm:PMD-MD-A})}
\label{sec:sample_appendix}

\begin{algorithm}[tb]
\caption{\textbf{Policy Mirror Descent-Primal Dual with Approximation (PMD-PD-A)}}
\label{alg:PMD-PD-A}
\noindent {\textbf Input:} $\rho, \alpha, \eta, \eta', \epsilon, \delta$; \\
\noindent {\textbf Initialization:} Let $\pi_0$ take a random action with a uniform distribution in every state, and $\lambda_{0, i} = \max\{0, -\eta' \hat{V}_{c_i}^{\pi_0}(\rho)\}, \forall i \in [m]$;\\
\For{$k = 0, 1, \dots, K-1$}{
\noindent {\bf\textit{[Inner loop (policy update)]}} \\
\noindent Take $\pi_k^{(0)} = \pi_k$ as the initialized policy and choose $t_k, M_{Q, k}, N_{Q, k}, M_{V, k+1}, N_{V, k+1}$ appropriately according to (\ref{eqn:params}); \\
\For{$t = 0, 1, \dots, t_k-1$}{
    \noindent Generate $M_{Q, k}$ independent trajectories of length $N_{Q, k}$ starting from any $(s, a) \in \Sc \times \Ac$ and estimate
    \begin{align*}
        \hat{Q}_{k,\alpha}^{\pi_k^{(t)}}(s, a) = \tilde{c}_{k}(s, a) + \alpha \log \frac{1}{\pi_k(a|s)} + \frac{1}{M_{Q, k}} \sum_{j=1}^{M_{Q, k}} \sum_{l=1}^{N_{Q, k}-1} \gamma^l \left[\tilde{c}_k(s_l^j, a_l^j) + \alpha \sum_{a'} \pi_k^{(t)}(a'|s_l^j) \log\frac{\pi_k^{(t)}(a'|s_l^j)}{\pi_k(a'|s_l^j)}\right];
    \end{align*}
    \noindent Update the policy according to the NPG updating formula
    \begin{align*}
        \pi_k^{(t+1)}(a|s) = \frac{1}{Z^{(t)}(s)} (\pi_k^{(t)}(a|s))^{1 - \frac{\eta \alpha}{1-\gamma}} \exp(\frac{-\eta \hat{Q}_{k, \alpha}^{\pi_k^{(t)}}(s, a)}{1-\gamma}), \forall (s, a) \in \Sc \times \Ac,
    \end{align*}
    where $Z^{(t)}(s) = \sum_{a' \in \Ac} (\pi_k^{(t)}(a'|s))^{1 - \frac{\eta \alpha}{1-\gamma}} \exp(\frac{-\eta \hat{Q}_{k, \alpha}^{\pi_k^{(t)}}(s, a')}{1-\gamma})$;
    }
\noindent $\pi_{k+1}(a|s) =\pi_k^{(t_k)}(a|s), \forall (s, a) \in \Sc \times \Ac;$ \\
\noindent {\bf\textit{[Outer loop (dual update)]}} \\
\noindent Generate $M_{V, k+1}$ independent trajectories of length $N_{V, k+1}$ starting from $s_0 \sim \rho$ and estimate
\begin{align*}
    \hat{V}_{c_i}^{\pi_{k+1}}(\rho) := \frac{1}{M_{V, k+1}} \sum_{j=1}^{M_{V, k+1}} \sum_{l=0}^{N_{V, k+1}-1} \gamma^l c_i(s_l^j, a_l^j), \ \forall i \in [m];
\end{align*}
$\lambda_{k+1, i} = \max\left\{-\eta' \hat{V}^{\pi_{k+1}}_{c_i}(\rho), \lambda_{k, i} + \eta' \hat{V}^{\pi_{k+1}}_{c_i}(\rho) \right\}$ for each $i = 1, 2, \ldots, m$;
}
\noindent {\textbf Output:} $\overline{\pi} = \frac{1}{K} \sum_{k=1}^{K} \pi_k$.
\end{algorithm}

We present the sample-based NPG-PD with approximation in Algorithm \ref{alg:PMD-PD-A}, and provide the proofs of Theorem \ref{thm:PMD-MD-A}. For a clear exposition of the analysis, we use the big-$O$, big-$\Theta$ and big-$\Omega$ notation by only focusing on the $\epsilon$ and $\delta$-dependent parameters. Throughout the analysis, we will use the following parameters.
\begin{equation}\label{eqn:params}
\begin{aligned}
    & K = \Theta\left(\frac{1}{\epsilon}\right), \quad 
    t_k = \Theta\left(\log\left(\frac{\max(1, \|\lambda_k\|_1)}{\epsilon}\right)\right), \quad
    \delta'_k = \Theta\left(\frac{\delta}{K t_k}\right), \\&
    M_{V, k} = \Theta\left(\frac{\log(1/\delta'_k)}{\epsilon^2}\right), \quad 
    N_{V, k} = \Theta\left(\log_{1/\gamma}\left(\frac{1}{\epsilon} \right) \right), \\&
    M_{Q, k} = \Theta\left(\frac{(\max(1,\|\lambda_k\|_1) + \epsilon t_k )}{\epsilon^2}\log(1/\delta'_k) \right), \quad 
    N_{Q, k} = \Theta\left(\log_{1/\gamma}\left( \frac{\max(1, \|\lambda_k\|_1)}{\epsilon}\right)\right).
\end{aligned}
\end{equation}
Lemma \ref{lem:dual-bound} gives a high probability bound that shows $\|\lambda_k\|_1 = \Oc(1)$. Noting $\|\lambda_k\|_1 = \Oc(1)$ in the above parameter assignments leads to the order of parameters shown in the proof idea of Theorem \ref{thm:PMD-MD-A}.

\subsection{Estimation and Concentration}

We first introduce ``good" events, conditioned on which the remaining analysis is carried out.
\begin{definition}[Good events]
For any macro step $k$, define a ``good event'' $\Ec_k := \bigcap_{j = 0}^k (\Ec_{V, j} \cap \Ec_{Q, j})$, where  
\begin{align*}
  \Ec_{V, j} & := \left\{ |\hat{V}_{c_i}^{\pi_{j}}(\rho) - V_{c_i}^{\pi_j}(\rho) | \leq \Oc(\epsilon), \forall i \in [m] \right\}, \\
    \Ec_{Q, j} & := \left\{\left|\hat{Q}_{j,\alpha}^{\pi_j^{(t)}}(s, a) - \tilde{Q}_{j,\alpha}^{\pi_j^{(t)}}(s, a)\right| \leq \Oc(\epsilon), \forall t = 0, \ldots, t_j - 1, \forall (s, a) \in \Sc \times \Ac \right\}.
\end{align*}
\end{definition}
The following lemma shows that the good events are also high probability events.
\begin{lemma}\label{lem:good}
Under the parameter assignments in (\ref{eqn:params}), $\Ec_{K-1}$ holds with probability $1 - \Oc(\delta)$.
\end{lemma}
\begin{proof}[Proof of Lemma \ref{lem:good}]
Denote $\Hc_{k, t}$ as the $\sigma$-algebra generated by the samples (random variables) acquired before the $t$-th step in the inner loop of the $k$-th outer loop. Let $\Hc_{k+1, -1} := \Hc_{k, t_{k}-1}$ and let $\Hc_{0, -1}$ be the trivial sigma-algebra. We know $\Ec_{V, k} \in \Hc_{k, 0}$ and $\Ec_{Q, k} \in \Hc_{k, t_k} = \Hc_{k+1, -1}$.

We first consider the concentration of the estimator $\hat{V}^{\pi_k}_{c_i}(\rho)$ conditioned on $\Hc_{k, -1}$. Recall
\begin{align*}
    \hat{V}_{c_i}^{\pi_{k+1}}(\rho) := \frac{1}{M_{V, k+1}} \sum_{j=1}^{M_{V, k+1}} \sum_{l=0}^{N_{V, k+1}-1} \gamma^l c_i(s_l^j, a_l^j), \ \forall i \in [m].
\end{align*}
Note that 
$| \sum_{l=0}^{N_{V, k}} \gamma^{l} c_i(s_l, a_l)| \leq \frac{1}{1 - \gamma}$. By Hoeffding's inequality (Lemma \ref{lem:hoeffding}), with $M_{V, k} = \Theta(\log(1/\delta'_k)/\epsilon^2)$, we can guarantee that $\hat{V}^{\pi_k}_{c_i}(\rho)$ concentrated around $\Eb[\hat{V}_{c_i}^{\pi_{k}}(\rho)| \Hc_{k, -1}]$ with precision $\epsilon$ with probability $1 - \Oc(\delta'_k)$. 
We can also verify that  $|\Eb[\hat{V}_{c_i}^{\pi_{k}}(\rho)| \Hc_{k, -1}] - V_{c_i}^{\pi_k}(\rho) | \leq \frac{\gamma^{N_{V, k}}}{1-\gamma}$. By the choice of $N_{V, k} = \Theta(\log_{1/\gamma}(1/\epsilon))$, we know $|\hat{V}_{c_i}^{\pi_{k}}(\rho) - V_{c_i}^{\pi_k}(\rho) | \leq \Oc(\epsilon)$ with probability $1 - \Oc(\delta'_k)$. 

We then study the concentration of the estimator $\hat{Q}_{k,\alpha}^{\pi_k^{(t)}}(s, a)$. Recall 
\begin{align*}
    \hat{Q}_{k,\alpha}^{\pi_k^{(t)}}(s, a) = \tilde{c}_{k}(s, a) + \alpha \log \frac{1}{\pi_k(a|s)} + \frac{1}{M_{Q, k}} \sum_{j=1}^{M_{Q, k}} \sum_{l=1}^{N_{Q, k}-1} \gamma^{l} \left[\tilde{c}_k(s_l^j, a_l^j) + \alpha \sum_{a'} \pi_k^{(t)}(a'|s_l^j) \log\frac{\pi_k^{(t)}(a'|s_l^j)}{\pi_k(a'|s_l^j)}\right].
\end{align*}
We know
\begin{align*}
    \left|\tilde{c}_k(s, a)\right| = \left|c_0(s, a) + \sum_{i=1}^m (\lambda_{k, i} + \eta' \hat{V}_{c_i}^{\pi_k}(\rho)) c_i(s, a)\right| \le \Oc\left(\max(1, \|\lambda_k\|_1) \right).
\end{align*}
 
For $t = 0$, by the same argument as in the concentration of $\hat{V}_{c_i}^{\pi_k}(\rho)$, we can similarly prove that choosing $M_{Q, k} = \Theta(\frac{\max(1,\|\lambda_k\|_1) + \epsilon t_k }{\epsilon^2}\log(1/\delta'_k))$ and $N_{Q, k} = \Theta(\log_{1/\gamma}( \max(1, \|\lambda_k\|)/\epsilon))$, gives $|\hat{Q}_{k,\alpha}^{\pi_k^{(0)}}(s, a) - \tilde{Q}_{k,\alpha}^{\pi_k^{(0)}}(s, a)| \leq \Oc(\epsilon)$ with probability $1 - \Oc(\delta'_k)$. For $t > 0$, we will prove inequality $|\hat{Q}_{k,\alpha}^{\pi_k^{(t)}}(s, a) - \tilde{Q}_{k,\alpha}^{\pi_k^{(t)}}(s, a)| \leq \Oc(\epsilon)$ holds conditioned on event $\Ec_{V, k}$. Assuming the inequality holds for $1$ to $t-1$, we know
\begin{align*}
    \alpha \sum_{a'} \pi_k^{(t)}(a'|s) \log\frac{\pi_k^{(t)}(a'|s)}{\pi_k(a'|s)} & \leq \tilde{V}^{\pi_k^{(t)}}_{k, \alpha}(s) - \Eb\left[\sum_{l = 0}^\infty \gamma^l \tilde{c}_k(s, a) | s_0 = s \right]\\
    & \leq \tilde{V}^{\pi_k^{(t)}}_{k, \alpha}(s) + \Oc\left(\max(1, \|\lambda_k\|_1)\right) \\
    & \stackrel{(a)}{\leq} \tilde{V}^{\pi_k}_{k, \alpha}(s) + \Oc\left( \max(1, \|\lambda_k\|_1)\right) + \sum_{l=0}^{t-1} \frac{2}{1 - \gamma} \|\hat{Q}^{\pi_k^{(l)}}_{k, \alpha} - Q^{\pi_k^{(l)}}_{k, \alpha} \|_\infty \\
    & \stackrel{(b)}{\leq} \Oc(\max(1, \|\lambda_k\|_1) + t \epsilon) = \Oc(\max(1, \|\lambda_k\|_1) + \epsilon t_k),
\end{align*}
where $(a)$ is obtained by iteratively applying Lemma \ref{lem:cen_perform}, and $(b)$ is true since $\tilde{V}^{\pi_k}_{k, \alpha}(s) = \Theta(\max(1, \|\lambda_k\|_1))$. By the choice of $t_k = \Theta(\log(\max(1, \|\lambda_k\|_1)/\epsilon))$, choosing $M_{Q, k} = \Theta\left(\frac{\max(1,\|\lambda_k\|_1) + \epsilon t_k }{\epsilon^2}\log(1/\delta'_k) \right)$ and $N_{Q, k} = \Theta\left(\log_{1/\gamma}\left( \frac{\max(1, \|\lambda_k\|_1)}{\epsilon}\right)\right)$ gives $|\hat{Q}_{k,\alpha}^{\pi_k^{(t)}}(s, a) - \tilde{Q}_{k,\alpha}^{\pi_k^{(t)}}(s, a)| \leq \Oc(\epsilon)$ with probability $1 - \Oc(\delta'_k)$.

We can conclude the proof by union bound, that $\Ec_{K-1}$ holds with probability at least $1 - \Oc(\sum_{k=0}^{K-1} t_k \delta'_k) = 1 - \Oc(\delta)$.
\end{proof}

\subsection{Proofs of Theorem \ref{thm:PMD-MD-A}}
\paragraph{Inner loop analysis.} The goal of the inner loop in macro step $k$ is to approximately solve the MDP with value $\tilde{V}^{\pi}_{k, \alpha}(\rho)$. Let $\pi^*_{k} \in \argmin_{\pi} \tilde{V}^{\pi}_{k, \alpha}(\rho)$ be an optimal policy.

We then perform a similar inner loop and outer loop analysis as we did for the oracle-based PMD-PD algorithm.
\begin{lemma}\label{lem:push-app-A}
Let $\alpha, \eta, \eta', K, t_k, M_{Q, k}, M_{V, k+1}$ be the same values as in Theorem \ref{thm:PMD-MD-A}. Then for any $k = 0, 1, \ldots, K-1$, and any policy $\pi$, conditioned on event $\Ec_{k}$,
\begin{align*}
    \tilde{V}_k^{\pi_{k+1}}(\rho) + \frac{\alpha}{1 - \gamma} D_{d^{\pi_{k+1}}_{\rho}}(\pi_{k+1} || \pi_k) & \leq \tilde{V}_k^{\pi}(\rho) + \frac{\alpha}{1 - \gamma} D_{d^{\pi}_{\rho}}(\pi || \pi_k) - \frac{\alpha}{1 - \gamma} D_{d^{\pi}_{\rho}}(\pi || \pi_{k+1}) + \Theta(\epsilon).
\end{align*}
\end{lemma}
\begin{proof}
The proof follows the same steps of Lemma \ref{lem:push-app}, but utilizes Lemma \ref{lem:cen-app} instead of Lemma \ref{lem:cen} by $K = \Theta(1/\epsilon)$.
\end{proof}

\paragraph{Outer loop analysis} 
The main objective of the analysis of the outer loop is to study the inner product term $\langle \lambda_k + \eta' \hat{V}^{\pi_{k}}_{c_{1:m}}(\rho),  V^{\pi_{k+1}}_{c_{1:m}}(\rho) \rangle$ in the Lagrangian. Define $\epsilon_k := \hat{V}_{c_{1:m}}^{\pi_k}(\rho) - V_{c_{1:m}}^{\pi_k}(\rho)$, then we have
\begin{align}
\label{eqn:outer_loop_A}
    & \langle \lambda_k + \eta' \hat{V}^{\pi_{k}}_{c_{1:m}}(\rho),  V^{\pi_{k+1}}_{c_{1:m}}(\rho) \rangle 
    = \langle \lambda_k, \hat{V}^{\pi_{k+1}}_{c_{1:m}}(\rho)\rangle + \eta' \langle V^{\pi_{k}}_{c_{1:m}}(\rho), V^{\pi_{k+1}}_{c_{1:m}}(\rho)\rangle + \langle \lambda_k, -\epsilon_{k+1} \rangle + \eta' \langle \epsilon_k, V^{\pi_{k+1}}_{c_{1:m}}(\rho) \rangle.
\end{align}
We first introduce a lemma summarizing the properties of dual variables.
\begin{lemma}
\label{lem:L_property-A}
Based on the definition of Lagrange multipliers in Algorithm \ref{alg:PMD-PD-A}, we have
\begin{enumerate}[itemsep=0pt]
    \item For any macro step k, $\lambda_{k, i} \ge 0, \ \forall i \in [m]$.
    \item For any macro step k, $\lambda_{k, i} + \eta' \hat{V}_{c_i}^{\pi_k}(\rho) \ge 0, \ \forall i \in [m]$.
    \item For macro step 0, $\|\lambda_{0, i}\|^2 \le \|\eta' \hat{V}_{c_i}^{\pi_0}(\rho)\|^2$; for any macro step $k > 0$, $\|\lambda_{k, i}\|^2 \ge \|\eta' \hat{V}_{c_i}^{\pi_k}(\rho)\|^2$, $\forall i \in [m]$.
\end{enumerate}
\end{lemma}
\begin{proof}
Since the proof only requires algebraic relations between $\lambda$ and $\hat{V}$, we can follow the same steps as in the proof of Lemma \ref{lem:L_property}, by replacing $V_{c_i}^{\pi}$ with $\hat{V}_{c_i}^{\pi}$.
\end{proof}

\begin{lemma}
\label{lem:lower_bound-A}
For any $k = 0, 1, \ldots, K-1$,
\begin{align*}
    \langle \lambda_{k}, \hat{V}^{\pi_{k+1}}_{c_{1:m}}(\rho) \rangle & \geq \frac{1}{2\eta'} \| \lambda_{k+1} \|^2 - \frac{1}{2\eta'}\| \lambda_{k} \|^2 - \eta' \| \hat{V}^{\pi_{k+1}}_{c_{1:m}}(\rho) \|^2 \\
    & = \frac{1}{2\eta'} \| \lambda_{k+1} \|^2 - \frac{1}{2\eta'}\| \lambda_{k} \|^2 - \eta' \left(\| V^{\pi_{k+1}}_{c_{1:m}}(\rho) \|^2 + \|\epsilon_{k+1}\|^2 + 2\langle V^{\pi_{k+1}}_{c_{1:m}}(\rho), \epsilon_{k+1} \rangle \right).
\end{align*}
\end{lemma}
\begin{proof}
The proof can follow the same steps as in the proof of Lemma \ref{lem:lower_bound} by replacing $V_{c_i}^{\pi}(\rho)$ with $\hat{V}_{c_i}^{\pi}(\rho)$.
\end{proof}
\begin{proof}[\textbf{Proof of Theorem \ref{thm:PMD-MD-A}: Optimality gap bound}] 
Here, we give the proof of the optimality gap bound (\ref{eqn:sample-gap}).

Taking $\pi = \pi^*$ in Lemma \ref{lem:push-app-A}, since $\lambda_{k, i} + \eta' \hat{V}^{\pi_k}_{c_i}(\rho) \geq 0$ by the second property in Lemma \ref{lem:L_property-A} and $V^{\pi^*}_{c_i}(\rho) \leq 0$ for any $i \in [m]$, we have
\begin{equation}
\label{eqn:inner_analysis-A}
\begin{aligned}
    &V_{c_0}^{\pi_{k+1}}(\rho) + \left\langle \lambda_k + \eta' \hat{V}^{\pi_{k}}_{c_{1:m}}(\rho),  V^{\pi_{k+1}}_{c_{1:m}}(\rho) \right\rangle +  \frac{\alpha}{1 - \gamma} D_{d^{\pi_{k+1}}_{\rho}}\left(\pi_{k+1} || \pi_k\right) \\
    \leq&  V_{c_0}^{\pi^*}(\rho) + \frac{\alpha}{1 - \gamma} D_{d^{\pi^*}_{\rho}}(\pi^* || \pi_k) - \frac{\alpha}{1 - \gamma} D_{d^{\pi^*}_{\rho}}(\pi^* || \pi_{k+1}) + \Theta(\epsilon).
\end{aligned}
\end{equation}

Plugging (\ref{eqn:decompose_inner}), (\ref{eqn:decompose_inner-2}), and Lemma \ref{lem:lower_bound-A} into (\ref{eqn:outer_loop_A}) leads to
\begin{equation}
\begin{aligned}
    & \left\langle \lambda_k + \eta' \hat{V}^{\pi_{k}}_{c_{1:m}}(\rho),  V^{\pi_{k+1}}_{c_{1:m}}(\rho) \right\rangle \geq \frac{1}{2 \eta'}(\|\lambda_{k+1}\|^2 - \|\lambda_{k}\|^2) + \frac{\eta'}{2}( \|V_{c_{1:m}}^{\pi_k}\|^2 - V_{c_{1:m}}^{\pi_{k+1}}\|^2)\\
    & \quad\quad + \langle \lambda_k, -\epsilon_{k+1} \rangle + \eta' \langle \epsilon_k, V^{\pi_{k+1}}_{c_{1:m}}(\rho)\rangle - \eta' \|\epsilon_{k+1}\|^2 - 2 \eta'\langle V^{\pi_{k+1}}_{c_{1:m}}(\rho), \epsilon_{k+1} \rangle -\frac{\gamma^2 m \eta'}{(1 - \gamma)^4} D_{d^{\pi_{k+1}}_\rho}(\pi_{k+1} || \pi_k).
\label{eqn:inner-product-bound-app}
\end{aligned}
\end{equation}

Substituting the lower bound of inner product $\langle \lambda_k + \eta' \hat{V}^{\pi_{k}}_{c_{1:m}}(\rho),  V^{\pi_{k+1}}_{c_{1:m}}(\rho) \rangle$ in (\ref{eqn:inner-product-bound-app}) into (\ref{eqn:inner_analysis-A}) leads to
\begin{align*}
    & V_{c_0}^{\pi_{k+1}}(\rho) + \frac{1}{2\eta'} \left(\|\lambda_{k+1}\|^2 - \|\lambda_{k}\|^2\right) + \frac{\eta'}{2}\left(\|V^{\pi_{k}}_{c_{1:m}}(\rho)\|^2 - \|V^{\pi_{k+1}}_{c_{1:m}}(\rho)\|^2\right) + \frac{\alpha (1-\gamma)^3 - \gamma^2 m \eta'}{(1 - \gamma)^4} D_{d^{\pi_{k+1}}_{\rho}}(\pi_{k+1} || \pi_k) \\
    \leq& V_{c_0}^{\pi^*}(\rho) + \frac{\alpha}{1 - \gamma} D_{d^{\pi^*}_{\rho}}(\pi^* || \pi_k) - \frac{\alpha}{1 - \gamma} D_{d^{\pi^*}_{\rho}}(\pi^* || \pi_{k+1}) + \epsilon_k.
\end{align*}
Define $\Delta_k := \Theta(\epsilon) + \langle \lambda_k, \epsilon_{k+1} \rangle - \eta' \langle \epsilon_k, V^{\pi_{k+1}}_{c_{1:m}}(\rho)\rangle + \eta' \langle \epsilon_{k+1} + 2 V^{\pi_{k+1}}_{c_{1:m}}(\rho), \epsilon_{k+1} \rangle $. When $\alpha = \frac{2 \gamma^2 m \eta'}{(1 - \gamma)^3}$, $\left( \frac{\alpha}{1 - \gamma} - \frac{\gamma^2 m \eta'}{(1 - \gamma)^4} \right) D_{d^{\pi_{k+1}}_{\rho}}(\pi_{k+1} || \pi_k) \geq 0$. Under event $\Ec_{K'-1}$, it follows from telescoping that
\begin{align}
    \sum_{k = 1}^{K'} V^{\pi_k}_{c_0}(\rho) & \leq K' V^{\pi^*}_{c_0}(\rho) + \frac{\alpha}{1 - \gamma} D_{d^{\pi^*}_\rho}(\pi^* || \pi_{0}) - \frac{\alpha}{1 - \gamma} D_{d^{\pi^*}_\rho}(\pi^* || \pi_{K'}) + \sum_{k=1}^{K'} \Delta_{k-1} \notag \\
    &\quad + \frac{\eta'}{2} \|V^{\pi_{K'}}_{c_{1:m}}(\rho)\|^2 - \frac{\eta'}{2} \|V^{\pi_{0}}_{c_{1:m}}(\rho)\|^2 + \frac{1}{2\eta'} \|\lambda_0\|^2 - \frac{1}{2\eta'} \|\lambda_{K'}\|^2 \label{eqn:sample_middle}\\
    & \stackrel{(a)}{\leq} K V^{\pi^*}_{c_0}(\rho) + \frac{\alpha}{1 - \gamma} D_{d^{\pi^*}_\rho}(\pi^* || \pi_{0}) - \frac{\alpha}{1 - \gamma} D_{d^{\pi^*}_\rho}(\pi^* || \pi_{K'}) + \frac{\eta'm}{(1-\gamma)^2} + \sum_{k=1}^{K'} \Delta_{k-1} \notag \\
    & \stackrel{(b)}{\leq} K' V_{c_0}^{\pi^*}(\rho) + \frac{\alpha}{1 - \gamma} \log(|\mathcal{A}|) + \frac{\eta'm}{(1-\gamma)^2} + \sum_{k=1}^{K'} \Delta_{k-1}, \label{eqn:sample_last}
\end{align}
where $(a)$ holds since $\|V^{\pi_{K}}_{c_{1:m}}(\rho)\|^2 \leq \frac{m}{(1-\gamma)^2}$ and $\|\lambda_0\|^2 \le \frac{\eta'^2 m}{(1-\gamma)^2}$, while $(b)$ holds since $\pi_0$ is the uniformly distributed policy and thus $D_{d^{\pi^*}_\rho}(\pi^* || \pi_{0}) = \sum_{s \in \Sc} d_{\rho}^{\pi^*}(s) \sum_{a \in \Ac}$ $\pi^*(a|s) \log (|\Ac| \pi^*(a|s)) \le \log(|\Ac|)$.  

When event $\Ec_{K-1}$ holds, $\Delta_k = \Oc(\max(1, \|\lambda_k\|_1) \epsilon)$. Letting $K' = K$ in (\ref{eqn:sample_last}) implies that 
\begin{align*}
    \frac{1}{K} \sum_{k = 1}^K \left(V^{\pi_k}_{c_0}(\rho) - V_{c_0}^{\pi^*}(\rho)\right) = \Oc\left(\epsilon + \max_{k}(\|\lambda_k\|_1) \epsilon\right).
\end{align*}
Applying Lemma \ref{lem:dual-bound} that $\|\lambda_k\|_1 = \Oc(1)$ concludes the optimality gap bound.
\end{proof}

\begin{lemma}\label{lem:dual-bound}
When event $\Ec_{K'-1}$ holds, $\|\lambda_{K'}\|_1 = \Oc(1)$.
\end{lemma}
\begin{proof}[Proof of Lemma \ref{lem:dual-bound}]
Consider the Lagrangian with optimal dual variable $L(\pi, \lambda^*) = V^{\pi}_{c_0}(\rho) + \sum_{i=1}^m \lambda^*_i V^{\pi}_{c_i}(\rho)$, whose minimum value $V^{\pi^*}_{c_0}(\rho)$ is achieved by the optimal policy $\pi^*$. We know
\begin{align*}
    & K' V^{\pi^*}_{c_0}(\rho) \stackrel{(a)}{=} K' L(\pi^*, \lambda^*) \leq \sum_{k = 1}^{K'} L(\pi_k, \lambda^*) = \sum_{k=1}^{K'} V^{\pi_k}_{c_0}(\rho) + \sum_{i = 1}^m \lambda^*_i \sum_{k=1}^{K'} V^{\pi_k}_{c_i}(\rho) \\
    \leq& \sum_{k=1}^{K'} V^{\pi_k}_{c_0}(\rho) + \frac{1}{\eta'} \sum_{i = 1}^m \lambda^*_i \lambda_{K', i} - \sum_{i = 1}^m \sum_{k = 1}^{K'} \lambda_i^* \epsilon_{k,i} \\
    \stackrel{(b)}{\leq}& K' V^{\pi^*}_{c_0}(\rho) + \frac{\alpha}{1 - \gamma} \left(D_{d^{\pi^*}_\rho}(\pi^* || \pi_{0}) - D_{d^{\pi^*}_\rho}(\pi^* || \pi_{K'})\right) + \frac{\eta'}{2} \|V^{\pi_{K'}}_{c_{1:m}}(\rho)\|^2 - \frac{1}{2\eta'} \|\lambda_{K'}\|^2\\
    &+ \sum_{k=1}^{K'} \Delta_{k-1} + \frac{1}{\eta'} \sum_{i = 1}^m \lambda^*_i \lambda_{K', i} - \sum_{i = 1}^m \sum_{k = 1}^{K'} \lambda_i^* \epsilon_{k,i} + \frac{m \eta'}{2(1-\gamma)^2}.
\end{align*}
$(a)$ holds due to  complementary slackness, and $(b)$ holds because of (\ref{eqn:sample_middle}) and $\|\lambda_0\|^2 \le \frac{\eta'^2 m}{(1-\gamma)^2}$. Note that
\begin{align*}
    & \frac{\alpha}{1 - \gamma} \left(D_{d^{\pi^*}_\rho}(\pi^* || \pi_{0}) - D_{d^{\pi^*}_\rho}(\pi^* || \pi_{K'})\right) + \frac{\eta'}{2} \|V^{\pi_{K'}}_{c_{1:m}}(\rho)\|^2  \\
    \leq& \frac{\alpha}{1 - \gamma} \log(|\Ac|) - \frac{(1-\gamma)^3 \alpha}{2 \gamma^2 m} \|V^{\pi_{K'}}_{c_{1:m}}(\rho) - V^{\pi^*}_{c_{1:m}}(\rho)\|^2 + \frac{\eta'}{2} \left\|\left(V^{\pi_{K'}}_{c_{1:m}}(\rho) - V^{\pi^*}_{c_{1:m}}(\rho)\right) + V^{\pi^*}_{c_{1:m}}(\rho)\right\|^2 \\
    =& \frac{\alpha}{1 - \gamma} \log(|\Ac|) + \left(\frac{\eta'}{2} - \frac{\gamma^2 m \eta'^2}{2[\gamma^2 m \eta' - (1 - \gamma)^3 \alpha]}\right) \|V^{\pi^*}_{c_{1:m}}(\rho)\|^2 \\
    &+ \frac{\gamma^2 m \eta' - (1-\gamma)^3 \alpha}{2 \gamma^2 m} \left\|V^{\pi_{K'}}_{c_{1:m}}(\rho) - V^{\pi^*}_{c_{1:m}}(\rho) + \frac{\gamma^2 m \eta'}{\gamma^2 m \eta' - (1-\gamma)^3 \alpha} V^{\pi^*}_{c_{1:m}}(\rho)\right\|^2.
\end{align*}
When $\alpha = \frac{2 \gamma^2 m \eta'}{(1 - \gamma)^3}$, $\frac{\gamma^2 m \eta' - (1-\gamma)^3 \alpha}{2 \gamma^2 m} \leq 0$ and $\frac{\eta'}{2} - \frac{\gamma^2 m \eta'^2}{2[\gamma^2 m \eta' - (1 - \gamma)^3 \alpha]} = \eta'$, we can derive that
\begin{align*}
    & \frac{1}{2\eta'} \|\lambda_{K'}\|^2 - \frac{1}{\eta'} \sum_{i = 1}^m \lambda^*_i \lambda_{K', i} \leq \frac{\alpha}{1 - \gamma} \log(|\Ac|) + \eta' \|V^{\pi^*}_{c_{1:m}}(\rho)\|^2 + \sum_{k=1}^{K'} \Delta_{k-1} - \sum_{i = 1}^m \sum_{k = 1}^{K'} \lambda_i^* \epsilon_{k,i}.
\end{align*}
It then follows that
\begin{align*}
    &\frac{1}{2\eta'}\|\lambda^* - \lambda_{K'}\|^2 = \frac{1}{2\eta'}\|\lambda^*\|^2 + \frac{1}{2\eta'} \|\lambda_{K'}\|^2 - \frac{1}{\eta'} \sum_{i = 1}^m \lambda^*_i \lambda_{K', i} \\
    \leq& \frac{1}{2\eta'} \|\lambda^*\|^2 + \frac{\alpha}{1 - \gamma} \log(|\Ac|) + \frac{\eta' m}{(1 - \gamma)^2} + \sum_{k=1}^{K'} \Delta_{k-1} - \sum_{i = 1}^m \sum_{k = 1}^{K'} \lambda_i^* \epsilon_{k,i},
\end{align*}
which implies
\begin{align*}
    \|\lambda^* - \lambda_{K'}\| & \leq \sqrt{ \|\lambda^*\|^2 + \frac{\alpha}{1 - \gamma} \log(|\Ac|) + \frac{m}{ (1 - \gamma)^2} + |\sum_{k=1}^{K'} \Delta_{k-1}| + | \sum_{i = 1}^m \sum_{k = 1}^{K'} \lambda_i^* \epsilon_{k,i}|}.
\end{align*}
Under event $\Ec_{K'-1}$, we know $|\sum_{k=1}^{K'} \Delta_{k-1}| = \Oc(1)$ and $| \sum_{i = 1}^m \sum_{k = 1}^{K'} \lambda_i^* \epsilon_{k,i}| = \Oc(1)$, which concludes the proof.
\end{proof}

\begin{proof}[\textbf{Proof of Theorem \ref{thm:PMD-MD-A}: Constraint violation bound}]
Here, we give the proof of the constraint violation bound (\ref{eqn:sample-vio}). For any $i \in [m]$, since $\lambda_{k, i} = \max\{-\eta' \hat{V}_{c_i}^{\pi_{k}}(\rho), \lambda_{k-1, i} + \eta' \hat{V}_{c_i}^{\pi_{k}}(\rho)\} \ge \lambda_{k-1, i} + \eta' \hat{V}_{c_i}^{\pi_{k}}(\rho)$, we have
\begin{align*}
    \sum_{k = 1}^{K} V^{\pi_k}_{c_i}(\rho) = \sum_{k = 1}^{K} \hat{V}^{\pi_k}_{c_i}(\rho) - \sum_{k = 1}^{K}\epsilon_{k, i}  \leq \frac{\lambda_{K, i}}{\eta'} - \sum_{k = 1}^{K} \epsilon_{k, i}.
\end{align*}
To analyze the constraint violation, it therefore suffices to bound the dual variables. By Lemma \ref{lem:dual-bound}, which gives an upper bound on the dual variables, under event $\Ec_{K - 1}$, the constraint violation bound in Theorem \ref{thm:PMD-MD-A} can be derived following
\begin{align*}
    & \frac{1}{K} \sum_{k = 1}^{K} V^{\pi_k}_{c_i}(\rho) \leq \frac{\lambda_{K, i}}{K \eta'} - 
    \frac{1}{K}\sum_{k = 1}^K \epsilon_{k, i} \stackrel{(a)}{\leq} \frac{\|\lambda^*\|}{K \eta'} + \frac{\|\lambda_{K} - \lambda^*\|}{K \eta'} - \frac{1}{K}\sum_{k = 1}^K \epsilon_{k, i} = \Oc(\epsilon),
\label{eqn:thm_second}
\end{align*}
where $(a)$ is by $\lambda_{K, i} \leq \|\lambda_{K}\|$ and $\|\lambda_K\| \leq \|\lambda^*\| + \|\lambda_K - \lambda^*\|$, $(b)$ is by Lemma \ref{lem:dual-bound}, $K = \Theta(1/\epsilon)$ and $-\sum_{k=1}^K \epsilon_{k, i} = \Oc(1)$ under $\Ec_{K-1}$.
\end{proof}

\section{Additional Experimental Results for Sample-based Algorithms}
\label{sec:exp_appendix}
In this section, we demonstrate the performance advantage of the sample-based PMD-PD algorithm (Algorithm \ref{alg:PMD-PD-A}) in the same tabular CMDP described  in Section \ref{sec:exp} and in a more complex environment Acrobot-v1 \cite{1606.01540}.

\subsection{Tabular CMDP}
We first consider the same tabular CMDP as described in in Section \ref{sec:exp}. Please note that in Figure \ref{fig:simple_cmdp} and Figure \ref{fig:simple_cmdp_log} in Section \ref{sec:exp}, we have already shown the faster global convergence of the PMD-PD algorithm ($\tilde{\Oc}(1/T)$) compared with CRPO \cite{xu2021crpo} and NPG-PD \cite{ding2020natural} ($\Oc(1/\sqrt{T})$), in the scenario where these algorithms have access to an oracle of exact policy evaluation. We here compare the performances of sample-based CRPO, NPG-PD and PMD-PD algorithms, where value functions are estimated from samples, and their performances are illustrated in Figure \ref{fig:simple_cmdp-A} which displays both the optimality gap and the constraint violation versus the number of iterations. The figure exhibits a faster convergence of the optimality gap of the sample-based PMD-PD algorithm, while all the three algorithms  satisfy the constraint after a short period of time.

\begin{figure}
\centering
\begin{subfigure}{0.49\textwidth}
\centering
\includegraphics[width=\textwidth]{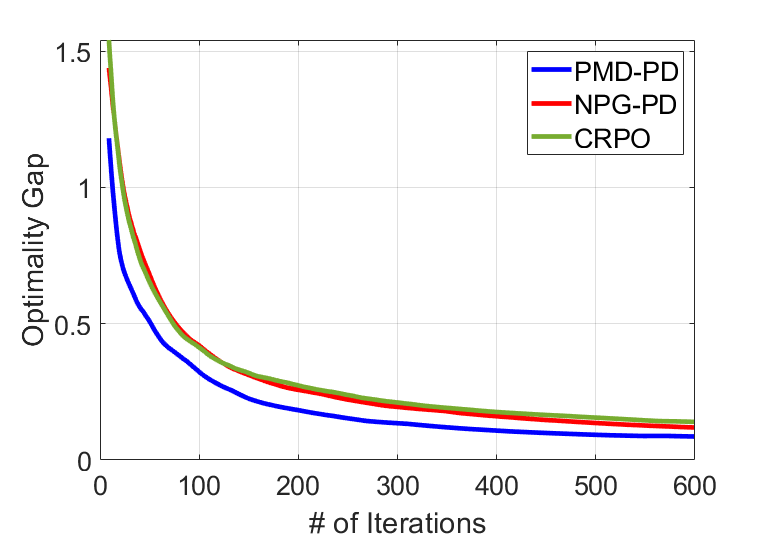}
\caption{}
\end{subfigure}
\begin{subfigure}{0.49\textwidth}
\centering
\includegraphics[width=\textwidth]{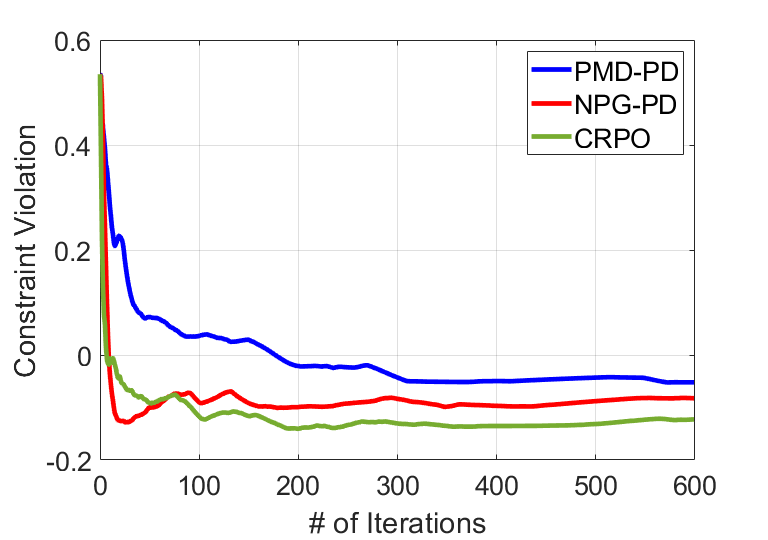}
\caption{}
\end{subfigure}
\caption{Optimality gap and constraint violation for sample-based PMD-PD, NPG-PD, and CRPO on a randomly generated CMDP with $|\Sc|=20, |\Ac|=10, \gamma=0.8, \text{ and } b=3$. Choose $\eta=1$ for all algorithms and $\eta'=1$ for sample-based NPG-PD and PMD-PD.}
\label{fig:simple_cmdp-A}
\end{figure}

\subsection{Acrobot-v1}
To demonstrate the performance of the PMD-PD algorithm on more complex tasks with a large state space and multiple constraints, we conduct experiments on the environment Acrobot-v1 from OpenAI Gym \cite{1606.01540}. The acrobot is a planar two-link robotic arm including two joints and two links, where the joint between the links is actuated. The objective is to swing the end of the lower link to a given height, while the two constraints are to apply torque on the joint (i) when the first link swings in a prohibited direction and (ii) when the second link swings in a prohibited direction with respect to the first link. 

For fairness of comparisons, all algorithms are based on the same neural softmax policy parameterization and the trust region policy optimization (TRPO) \cite{schulman2015trust}. Since TRPO is implemented via penalty and linear-quadratic approximation for the KL-divergence term, it is equivalent to the implementation of NPG. Given that the exact policy evaluation is no longer accessible for Acrobot-v1, we will adopt the sample-based versions for all algorithms, i.e., using empirical estimates of policy evaluation. Figure \ref{fig:acrobot} provides the average performance over 10 random seeds, where the best step size of the dual update (i.e., 0.0005) is tuned from the set $\{0.00001, 0.0005, 0.001, 0.005, 0.01, 0.05\}$.

Figure \ref{fig:acrobot}(a) shows that the PMD-PD algorithm has a larger accumulated reward compared with the NPG-PD algorithm, while Figures \ref{fig:acrobot}(b) and \ref{fig:acrobot}(c) illustrate that the accumulated cost of the PMD-PD algorithm is closer to the threshold after the constraints are satisfied. The closer gap to the threshold and the larger accumulated reward are attributed to the newly designed Lagrangian. After the constraints are satisfied, the newly designed Lagrangian focuses more on improving accumulated reward since $0 \le \lambda_{k, i} + \eta' ( V_{c_i}^{\pi_k}(\rho) - 50) \le \lambda_{k, i}, \forall k \in [K], \forall i \in [m]$. Note that the CRPO algorithm enjoys a faster convergence rate to thresholds for cost constraints, at the cost of a slower convergence rate to the optimal reward.

\begin{figure}
\centering
\begin{subfigure}{0.32\textwidth}
\centering
\includegraphics[width=\textwidth]{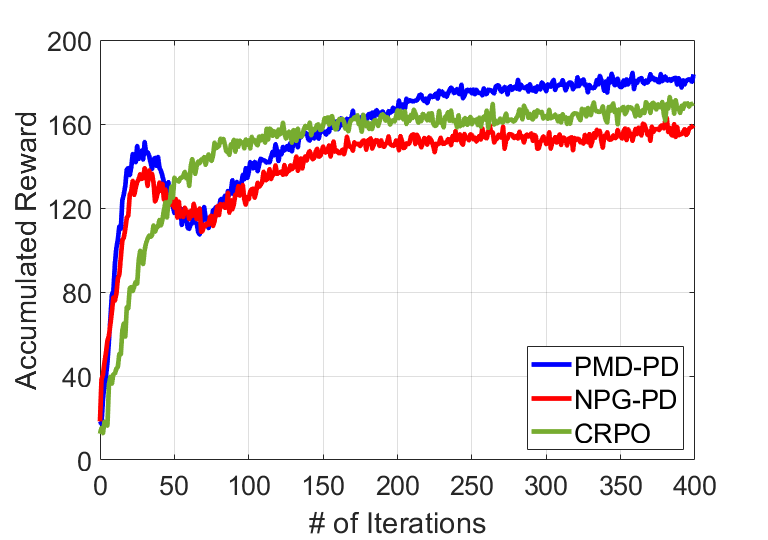}
\caption{}
\end{subfigure}
\begin{subfigure}{0.32\textwidth}
\centering
\includegraphics[width=\textwidth]{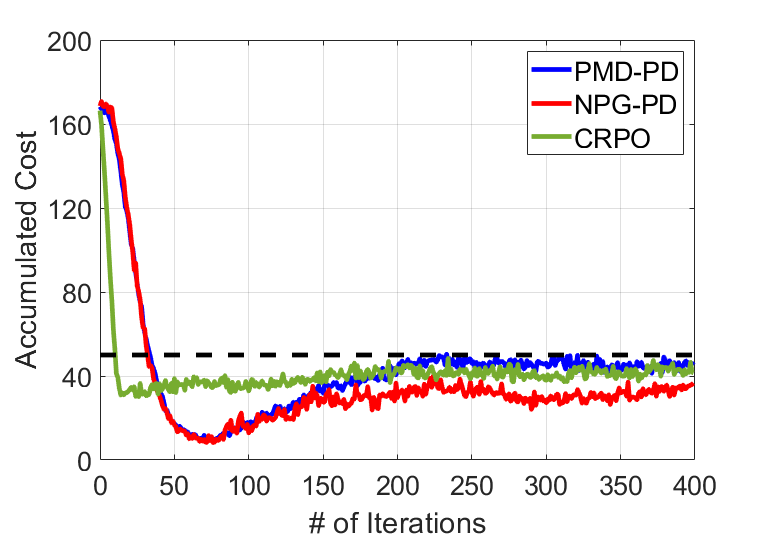}
\caption{}
\end{subfigure}
\begin{subfigure}{0.32\textwidth}
\centering
\includegraphics[width=\textwidth]{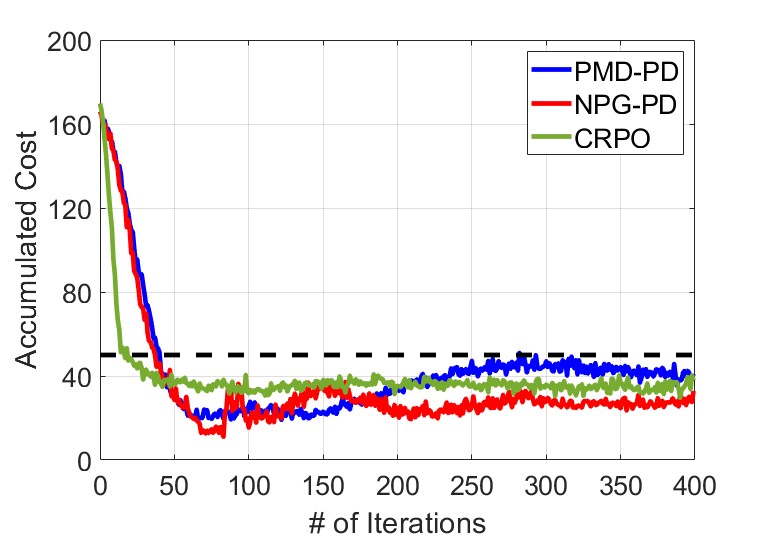}
\caption{}
\end{subfigure}
\caption{Average performance for sample-based PMD-PD, NPG-PD, and CRPO over 10 random seeds. The black dashed lines in (b) and (c) represent given thresholds (i.e., 50) for cost constraints. Choose $\eta$ via backtracking line search and $\eta'=0.0005$ for sample-based PMD-PD and NPG-PD.}
\label{fig:acrobot}
\end{figure}


\begin{thebibliography}{10}
\bibitem{xu2020first}
Y.~Xu.
\newblock First-order methods for problems with {O(1)} functional constraints
  can have almost the same convergence rate as for unconstrained problems.
\newblock {\em arXiv preprint arXiv:2010.02282}, 2020.
\end{thebibliography}
\end{document}